\documentclass[10pt,twocolumn,letterpaper]{article}

\usepackage{cvpr}              
\usepackage[accsupp]{axessibility} 

\usepackage[dvipsnames]{xcolor}

\usepackage{url}            
\usepackage{xcolor}         

\usepackage{times}
\usepackage{microtype} 
\usepackage{graphicx} 
\usepackage{wrapfig}
\usepackage{balance} 

\usepackage{algorithm,algpseudocode}

\usepackage{helvet}
\usepackage{courier}

\usepackage{makecell}
\usepackage{graphicx}
\usepackage{color}
\usepackage{amsfonts}
\usepackage{amsmath}
\usepackage{amssymb}

\usepackage{bm}

\usepackage{multirow}
\usepackage{multicol} 
\usepackage{booktabs}

\usepackage{colortbl} 
\usepackage{caption} 
\usepackage[dvipsnames]{xcolor} 
\colorlet{colorFst}{Green!25}       
\colorlet{colorSnd}{SpringGreen!45} 
\colorlet{colorTrd}{Yellow!30}      
\colorlet{colorLow}{darkgray!30}    
\colorlet{colorDeg}{Orange!30}      
\newcommand{\fs}{\cellcolor{colorFst}\bf}   
\newcommand{\nd}{\cellcolor{colorSnd}}      
\newcommand{\rd}{\cellcolor{colorTrd}}      

\DeclareMathOperator*{\argmin}{arg\,min}

\def\ie{\mbox{\textit{i.e.}, }}
\def\eg{\mbox{\textit{e.g.}, }}
\def\wrt{\mbox{\textit{w.r.t. }}}

\def\balpha{\mbox{{\boldmath $\alpha$}}}

\def\bmu{\mbox{{\boldmath $\mu$}}}

\def\beps{\mbox{{\boldmath $\epsilon$}}}

\def\bnu{\mbox{{\boldmath $\nu$}}}

\def\mL{{\mathcal L}}

\def\mN{{\mathcal N}}

\def\mR{{\mathcal R}}

\DeclareMathAlphabet\mathbfcal{OMS}{cmsy}{b}{n}

\def\0{{\bf 0}}
\def\1{{\bf 1}}

\def\bA{{\bm{A}}}

\def\bG{{\bm{G}}}

\def\bI{{\bm{I}}}

\def\bM{{\bm{M}}}

\def\bn{{\bm n}}

\def\bx{{\bm x}}
\def\by{{\bm y}}
\def\bz{{\bm z}}

\def\balpha{{\bm \alpha}}

\def\bmu{{\bm \mu}}

\usepackage{ntheorem}

\newtheorem*{*thm}{Theorem}
\newtheorem{prop}{Proposition}

\newtheorem*{*lemma}{Lemma}

\newenvironment*{proof}{\textbf{Proof}\quad}{\hfill $\square$\par\vspace{1mm}}

\usepackage{stackengine}
\usepackage{tabularx}
\usepackage{threeparttable}
\usepackage{arydshln}

\definecolor{mygray}{gray}{.9}

\usepackage{pifont}

\usepackage{adjustbox}

\newlength \g

\usepackage{enumitem}

\definecolor{cvprblue}{rgb}{0.21,0.49,0.74}
\usepackage[pagebackref,breaklinks,colorlinks,citecolor=cvprblue]{hyperref}

\def\paperID{11512} 
\def\confName{CVPR}
\def\confYear{2024}

\def\mytitle{Deep Equilibrium Diffusion Restoration with Parallel Sampling}
\title{\mytitle}

\author{Jiezhang Cao$^{1}$,\enspace Yue Shi$^{1,2}$,\enspace Kai Zhang$^{3,}$\thanks{Corresponding author.}~,\enspace Yulun Zhang$^{2}$,\enspace Radu Timofte$^{1,4}$,\enspace Luc Van Gool$^{1,5}$\\
\!\!\!\!\!\!\!\textsuperscript{1}ETH Z\"{u}rich, \!\!
\textsuperscript{2}Shanghai Jiao Tong University, \!\!
\textsuperscript{3}Nanjing University, \!\!
\textsuperscript{4}University of W\"{u}rzburg, \!\!
\textsuperscript{5}KU Leuven \\
{\small \textbf{\url{https://github.com/caojiezhang/DeqIR}}}
}

\begin{document}
\maketitle
\begin{abstract}
    Diffusion model-based image restoration (IR) aims to use diffusion models to recover high-quality (HQ) images from degraded images, achieving promising performance. Due to the inherent property of diffusion models, most existing methods need long serial sampling chains to restore HQ images step-by-step, resulting in expensive sampling time and high computation costs. Moreover, such long sampling chains hinder understanding the relationship between inputs and restoration results since it is hard to compute the gradients in the whole chains. In this work, we aim to rethink the diffusion model-based IR models through a different perspective, \ie a deep equilibrium (DEQ) fixed point system, called \textbf{DeqIR}. Specifically, we derive an analytical solution by modeling the entire sampling chain in these IR models as a joint multivariate fixed point system. Based on the analytical solution, we can conduct parallel sampling and restore HQ images without training. Furthermore, we compute fast gradients via DEQ inversion and found that initialization optimization can boost image quality and control the generation direction. Extensive experiments on benchmarks demonstrate the effectiveness of our method on typical IR tasks and real-world settings. 
\end{abstract}    
\section{Introduction}
\label{sec:intro}
Image restoration (IR) aims at recovering a high-quality (HQ) image from a degraded input. 
Recently, diffusion models \cite{ho2020ddpm,song2020ddim} are attracting great attention because they can generate higher quality images than GANs \cite{dhariwal2021diffusion} and likelihood-based models \cite{kingma2021variational}.
Based on diffusion models \cite{ho2020ddpm, song2020ddim}, many IR methods \cite{kawar2022ddrm,chung2023dps,wang2022ddnm} achieve compelling performance on different tasks.
Directly using diffusion models in IR, however, suffers from some limitations.

\begin{figure}[t]
    \centering
    \includegraphics[width=1.00\linewidth]{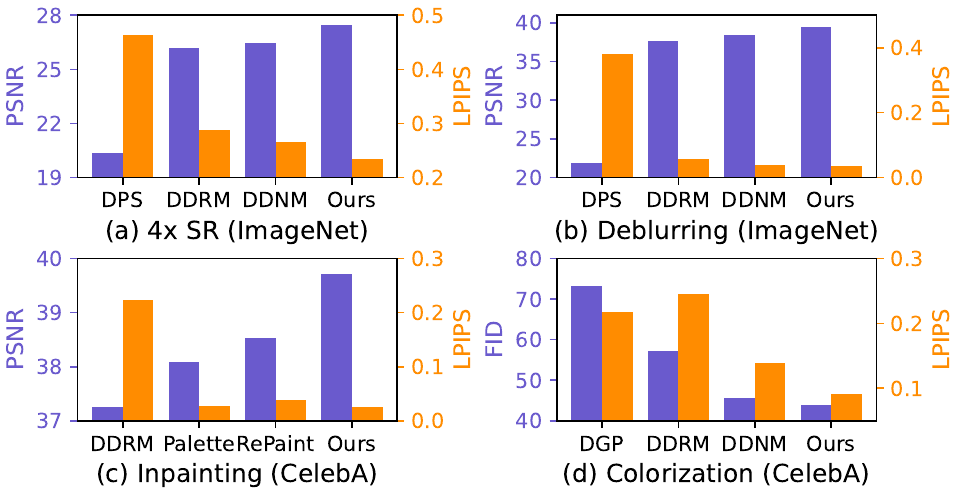}
   \vspace{-8mm}
   \caption{Comparisons of different zero-shot DMIR methods in various IR applications on different datasets.}
   \vspace{-6mm}
   \label{fig:comp_ir}
\end{figure}

First, diffusion model-based image restoration (DMIR) models rely on a long sampling chain to synthesize HQ images step-by-step, as shown in Figure \ref{fig:comp} (a).
As a result, it will lead to expensive sampling time during the inference.
For example, DPS \cite{chung2023dps} based on DDPM \cite{ho2020ddpm} needs 1k sampling steps.
To accelerate the sampling, some DMIR methods \cite{kawar2022ddrm,wang2022ddnm,zhu2023diffpir} use DDIM \cite{song2020ddim} to make a trade-off between computational cost and the restoration quality.
Based on this, these methods can reduce sampling steps to 100 or even fewer.
Unfortunately, it may degrade the sample quality when reducing the sampling steps \cite{lu2022dpm}.
It raises an interesting question: \emph{is it possible to develop an alternative sampling method without sacrificing the sample quality?}

Second, the long sampling chain makes understanding the relationship between the restoration and inputs difficult.
In practice, sampling different Gaussian noises as inputs may have diverse results for some IR tasks (\eg inpainting and colorization).
Such diversity is not necessary for some IR tasks, \eg super-resolution (SR) or deblurring.
Nevertheless, different initializations may affect the quality of SR and deblurring. 
It raises the second question: \emph{is it possible to optimize the initialization such that the generation can be improved or controlled?}
However, it is difficult for existing methods to compute the gradient along the long sampling chain as they require storing the entire computational graph.

\begin{figure*}[t]
  \centering
   \includegraphics[width=0.95\linewidth]{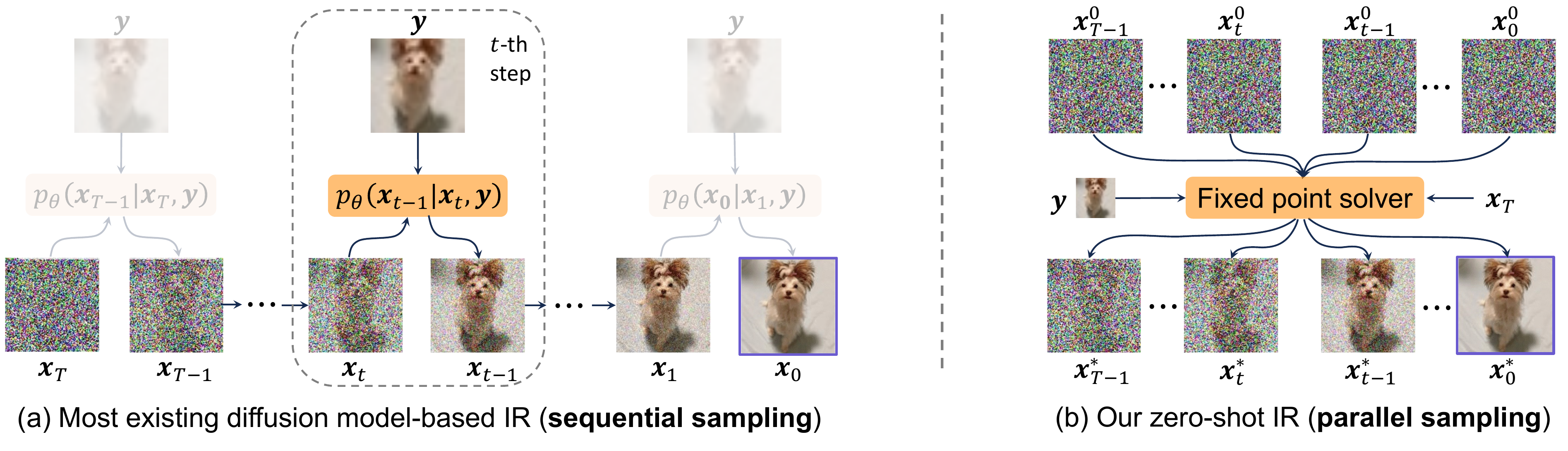}
   \vspace{-4mm}
   \caption{Comparisons of sequential sampling and our parallel sampling.}
   \vspace{-6mm}
   \label{fig:comp}
\end{figure*}

In this paper, we rethink the sampling process in IR from a deep equilibrium (DEQ) based on \cite{pokle2022deqddim}.
Specifically, we first derive a proposition to model the sampling chain as a fixed point system, achieving parallel sampling.
Then, we use a DEQ solver to find the fixed point of the sampling chain.
Last, we use modern automatic differentiation packages to compute the gradients with backpropagating and understand the relationship between input noise and restoration.

We summarize our contributions as follows:
\begin{itemize}
	\item We prove that the long sampling chain in DMIR can be formulated in a parallel way. Then we analytically formulate the generative process as a deep equilibrium fixed point system. Moreover, the generation has a convergence guarantee with few timesteps and iterations. 
    \vspace{1mm}
    \item Compared with most existing DMIR methods with sequential sampling, our method is able to achieve parallel sampling, as shown in Figure \ref{fig:comp} (b). Moreover, our method can be run on multiple GPUs instead of a single GPU.
    \vspace{1mm}
    \item Our model has more efficient gradients using DEQ inversion than existing DMIR methods which need a large computational graph for storing intermediate variables. The gradients can be computed through standard automatic differentiation packages.
    Moreover, we found that the initialization can be optimized with the gradients to improve the image quality and control the generation direction. 
    \vspace{1mm}
	\item Extensive experiments on benchmarks demonstrate the effectiveness of our zero-shot method on different IR tasks, as shown in Figure \ref{fig:comp_ir}. Moreover, our method performs well in real-world applications that may contain unknown and non-linear degradations.
\end{itemize}
\section{Related Work}
\vspace{-1.5mm}
\paragraph{Deep implicit learning (DIL).}
DIL attracts more and more attention and has emerging applications.
Different from explicit learning, DIL is based on dynamical systems, \eg optimization \cite{amos2017optnet,djolonga2017differentiable,donti2021dc3,geng2020attention,song2020score}, differential equation \cite{chen2018neural,dupont2019augmented,gu2021efficiently}, or fixed-point system \cite{bai2019deq,bai2020mdeq,gu2020implicit}.
For the fixed-point system, DEQ \cite{bai2019deq} is a new type of implicit model and it models sequential data by directly finding the fixed point and optimizing this equilibrium.
Recently, DEQ has been widely used in different tasks, \eg semantic segmentation \cite{bai2020mdeq}, object detection \cite{wang2023deep,wang2020implicit}, robustness \cite{yang2022closer,wei2021certified,li2022cerdeq}, optical flow estimation \cite{bai2022deep}, and generative models like normalizing flow \cite{lu2021implicit}.
Notably, DEQ-DDIM \cite{pokle2022deqddim} apply DEQs to diffusion models \cite{ho2020ddpm} by formulating this process as an equilibrium system.
However, applying DEQs in diffusion model-based IR methods is non-trivial because the generative process is complex, and formulating such a process is very challenging.

\vspace{1mm}
\noindent\textbf{Diffusion model-based image restoration.}
Previous image restoration (IR) methods \cite{dong2015srcnn,dong201arcnn} use convolutional neural networks (CNN) to achieve impressive performance on IR.
Up to now, many researchers propose to design the network architecture using residual blocks~\cite{VDSR,plug-denoiser,CAS-CNN}, GANs~\cite{WGAN-GP,ESRGAN,Real-ESRGAN,inpainting-GAN,menon2020pulse,cao2019mwgan,cao2018lccgan,cao2020improving,zhang2021bsrgan}, attention~\cite{RCAN,ENLCA,SAN,inpainting1,inpainting2,inpainting3,deepfillv2,IPT,SWINIR,restormer,MAT,NAFNet,liang2022vrt,cao2021vsrt,cao2023ciaosr,cao2022datsr,cao2022davsr}, and others~\cite{KDSR,restor12,restor13,restor14,restor15, AOTGAN} to improve the IR performance.

Recently, denoising diffusion probabilistic models (DDPM) \cite{ho2020ddpm} developed a powerful class of generative models that can synthesize high-quality images \cite{dhariwal2021diffusion} from noise step-by-step.
Based on the diffusion models, existing IR methods \cite{kawar2022ddrm,chung2023dps,wang2022ddnm} can be divided into supervised methods and zero-shot methods.
The supervised methods aim to train a conditional diffusion model in the image space \cite{saharia2022sr3,whang2022deblurring,saharia2022palette,yue2022difface} or the latent space \cite{rombach2022ldm,wang2023stablesr,xia2023diffir,lin2023diffbir}.
However, these methods need training diffusion models for the specific degradations and have limited generalization performance to other degradations in different IR tasks.

For zero-shot IR methods, they use a pre-trained diffusion model (\eg DDPM \cite{ho2020ddpm} and DDIM \cite{song2020ddim}) to restore images without training \cite{kawar2022ddrm,wang2022ddnm,chung2023dps}.
For example, based on a given reference image, ILVR \cite{choi2021ilvr} guides the generative process in DDPM and generates high-quality images.
Based on DDPM, DPS \cite{chung2023dps} solves the inverse problems via approximation of the posterior sampling using 1000 steps of the manifold-constrained gradient. 
Similar to DPS, DiffPIR \cite{zhu2023diffpir} integrates the traditional plug-and-play method into the diffusion models.
Repaint \cite{lugmayr2022repaint} also employs a pre-trained DDPM as the generative prior for the image inpainting task.
To accelerate the sampling, there are some IR methods using DDIM.
For example, DDRM \cite{kawar2022ddrm} applies a pre-trained denoising diffusion generative model to solve a linear inverse problem with 20 sampling steps. This method uses SVD on the degradation operator, which is similar to SNIPS \cite{kawar2021snips}. 
Based on SVD, DDNM \cite{wang2022ddnm} applies range-null space decomposition in linear image inverse problem and refines the null-space iteratively. Here, DDNM uses DDIM as the base sampling strategy with 100 sampling steps.
However, all of these methods use the serial sampling chain, resulting in a long sampling time and expensive computational cost.
\section{Preliminaries}

\paragraph{Image restoration.}
Image restoration aims at synthesizing high-quality image $\hat{\bx}$ from a degraded observation $\by = \bA \left({\bx} \right) + \bn_{\sigma} $, where $\bA$ is some degradation (\eg bicubic), $\bx$ is the original image, and $\bn_{\sigma}$ is a non-linear noise (\eg  white Gaussian noise) with the level $\sigma$.
The solution can be obtained by optimizing the following problem: 
\begin{align}
	\hat{\bx} = \argmin\nolimits_{\bx} {1}/{2\sigma^2} \left\| \bA(\bx) - \by \right\|_2^2 + \lambda \mR (\bx),
\end{align}
where $\mR (\bx)$ is a reguralization term with a trade-off parameter $\lambda$, \eg sparsity and Tikhonov regularization.

\vspace{1mm}
\noindent\textbf{Diffusion models.}
DDPM \cite{ho2020ddpm} is a generative model that can synthesize high-quality images with a forward process (\ie diffusion process) and a reverse process.
The forward process gradually introduces noise from Gaussian distribution $\mN(\cdot)$ with specific noise levels to the data, \ie
\begin{align}
    q(\bx_t | \bx_{0}) = \mN \left( \bx_t; \sqrt{\bar\alpha_t} \bx_{0}, (1-\bar\alpha_t) \bI \right),
\end{align}
where $\bar\alpha_t:=\Pi_{s=1}^t \alpha_s$, $\alpha_t:=1-\beta_t$ and $\beta_t$ is a variance.
For the reverse process, the previous state $\bx_{t-1}$ can be predicted with $\tilde{\bmu}_t$ and $\tilde{\sigma}_t$, which is formulated as:
\begin{align}
    q(\bx_{t-1} | \bx_{t}, \bx_{0}) = \mN \left( \bx_{t-1}; \tilde{\bmu}_t(\bx_t, \bx_0), \tilde{\sigma}_t^2 \bI \right),
\end{align}
where $\tilde{\bmu}_t(\bx_t, \bx_0):=\frac{\sqrt{\bar\alpha_{t-1}}\beta_t}{1-\bar\alpha_t} \bx_0 + \frac{\sqrt{\alpha_t}(1-\bar\alpha_{t-1})}{1-\bar\alpha_t} \bx_t=\frac{1}{\sqrt{{\alpha}_t} } ( \bx_t - \frac{1-\alpha_t}{\sqrt{1 - \bar{\alpha}_t}} \beps )$ 
and $\tilde{\sigma}_t^2 := \frac{1{-}\bar\alpha_{t{-}1}}{1{-}\bar\alpha_t} \beta_t$.
Here, the noise $\beps\sim \mN(0, \bI)$ can be estimated by $\beps_{\theta} (\bx_t, t)$ in each time-step.
To apply $\tilde{\bmu}_t$ to the image inverse problem, one can replace $\bx_0$ with $\hat{\bx}_{0|t}$ conditioned on the degraded image $\by$, \ie
\begin{align}
	\bx_{t-1} = \frac{\sqrt{\bar{\alpha}_{t-1}} \beta_t}{1 - \bar{\alpha}_t} \hat{\bx}_{0|t} + \frac{\sqrt{\bar{\alpha}_{t}} (1 - \bar{\alpha}_{t-1}) }{1 - \bar{\alpha}_t} \bx_t + \tilde{\sigma}_t \beps,
\end{align}
where $\hat{\bx}_{0|t}$ can be estimated by using a degradation $\bA$ to map the denoised image $\bx_{0|t} = \frac{1}{\sqrt{\bar{\alpha}_t}} ( \bx_t - \sqrt{1 -  \bar{\alpha}_t} \beps_{\theta} (\bx_t, t))$ in the degradation space \cite{wang2022ddnm}, \ie
\begin{align}
	\hat{\bx}_{0|t} = \bA^{\dag} \by + (\bI - \bA^{\dag} \bA)  \bx_{0|t},
\end{align}
where $\bA^{\dag}$ is the pseudo-inverse of $\bA$.

\vspace{1mm}
\noindent\textbf{Deep equilibrium models.}
Deep equilibrium models (DEQs) \cite{bai2019deq} are infinite depth feed-forward networks that can find fixed points in the forward pass. Given an input injection $\bx$, an hidden state $\bnu^{k+1}$ can be predicted by using an equilibrium layer $f_{\theta}$ parametrized by $\theta$, \ie
\begin{align}
	\bnu^{k+1} = f_{\theta} \left( \bnu^{k}; \bx \right), k=0, \ldots, L{-}1.
\end{align}
When increasing the depth towards infinity, the model tends to converge to a fixed point (equilibrium) $\bnu^*$, \ie
\begin{align}
	\lim\limits_{k\to \infty} f_{\theta} \left( \bnu^{k}; \bx \right) = f_{\theta} \left( \bnu^*; \bx \right) = \bnu^*.
\end{align}
To solve the equilibrium state $\bnu^*$, one can use some fixed point solvers, like Broyden's method \cite{broyden1965class}, or Anderson acceleration \cite{anderson1965iterative}, and it can be accelerated by the neural solver \cite{chen2018neural} in the inference.

\section{Methodology}
\subsection{Deep Equilibrium Diffusion Restoration}
Most existing zero-shot IR methods \cite{chung2023dps,kawar2022ddrm,wang2022ddnm} restore high-quality images step-by-step with long serial sampling chains.
Such an inherent property comes from the diffusion models, and it will lead to expensive sampling time and high computation costs.
This issue may be intractable if we need a gradient by backpropagating through the long sampling chains which often result in out-of-memory in the experiments.
To address this issue, we present a main modeling contribution in this paper.

\vspace{1mm}
\noindent\textbf{Fixed point modeling.}
Motivated by \cite{pokle2022deqddim}, our goal is to formulate diffusion model-based IR as a deep equilibrium fixed point system. 
Specifically, given a degraded image $\by$ and Gaussian noise $\bx_T$, the sampling chain $\bx_{0: T-1}$ can be treated as multivariable of the DEQ fixed point system, we first formulate $\bx_{0:T{-}1}$ as follows:
\begin{align}
	\bx_{0: T-1} = F (\bx_{0: T-1}; (\bx_T, \by)), \label{eqn:deq1}
\end{align}
where $\bx_T\sim \mN(\0, \bI)$ and $\by$ are the input injections, and $F(\cdot)$ is a function that performs sequential data across all the sample steps simultaneously.
To formulate the function $F$ in Eqn. \eqref{eqn:deq1}, we first provide the following proposition for the parallel sampling.

\begin{prop} \emph{(\textbf{Parallel sampling})} \label{prop:parallel}
	Given a degradation matrix $\bA$, a degraded image $\by$ and a Gaussian noise image $\bx_T\sim \mN(\0, \bI)$, for $k \in [1, \ldots, T]$, the state $\bx_{T-k}$ can be predicted by previous states $\{ \bx_{T-k+1}, \ldots, \bx_{T} \}$, \ie
    \begin{equation}
	\begin{aligned}
		\label{eqn:deqir-xTk}
        \bx_{T-k} =& \frac{\sqrt{\bar{\alpha}_{T-k}}}{\sqrt{\bar{\alpha}_T} } \left(\bI - \bA^{\dagger} \bA \right) \bx_T + \bA^{\dagger} \bA \bz_{T-k+1} \\
        &+ \sum_{s=T-k}^{T-1} \frac{\sqrt{\bar{\alpha}_{T-k}}}{\sqrt{\bar{\alpha}_{s}} } \left(\bI - \bA^{\dagger} \bA \right) \bz_{s+1}, 
	\end{aligned}
    \end{equation}
    where $\bz_{s}= c_s^0 \beps_{\theta} (\bx_{s}, {s}) + \sqrt{\bar{\alpha}_{s-1}} \bA^{\dagger} \by + c_s^1 \beps_s $, the coefficients are defined as $c_s^0 := c_s^2 - \sqrt{(1 - \bar{\alpha}_s)/{\alpha}_s} (\bI - \bA^{\dagger} \bA )$, $c_s^1 := \sqrt{1 - \bar{\alpha}_s} \eta$ and $c_s^2 := \sqrt{1 - \bar{\alpha}_s} \sqrt{1-\eta^2}, 0 \leq \eta < 1$.
\end{prop}
\begin{proof}
    Please refer to the proofs in Supplementary.
\end{proof}
From the proposition, $\bx_{T-k}$ is related to subsequent states $\bx_{T-k+1:T}$ and the degraded image $\by$.
It means that our method is different from most existing diffusion model-based IR methods which update $\bx_t$ based only on $\bx_{t+1}$.
Based on our proposition, the timestep $T$ can be small using DDIM \cite{song2020ddim}.
In addition, the proposition can be extended to start from the intermediate state.
Motivated by \cite{yue2022difface}, we can predict the intermediate state using a restoration model (\eg \cite{SWINIR,chen2023activating,zhou2022codeformer,zhang2016colorful}) to provide prior information from the restoration model during the sampling processing when the degradation matrix $\bA$ is unknown or inaccurate.

Based on our proposed proposition, we can formulate the right side of Eqn. \eqref{eqn:deqir-xTk} as $\bx_{T-k}=f(\bx_{T-k+1:T}; \by)$. 
Then, we can write all sampling steps as a ``fully-lower-triangular'' inference process, \ie
\begin{align}
    \label{eqn:deq_triangular}
	\begin{bmatrix}
		\bx_{T{-}1}\\
		\bx_{T{-}2}\\
		\vdots\\
		\bx_{0}
	\end{bmatrix}  &=
	\begin{bmatrix}
		f(\bx_{T}; \by)\\
		f(\bx_{T{-}1:T}; \by)\\
		\vdots\\
		f(\bx_{1:T}; \by)\\
	\end{bmatrix},
\end{align}
where the function $f$ can be implemented in all sequential states in parallel, corresponding to Eqn. \eqref{eqn:deq1}, \ie $\bx_{0: T-1} = F (\bx_{0: T-1}; (\bx_T, \by) )$.
To find the solution to the fixed point of Eqn. \eqref{eqn:deq_triangular}, we apply commonly used fixed point solvers like Anderson acceleration \cite{anderson1965iterative} which can accelerate the convergence of the fixed-point sequence. 
To this end, we first define the residual $g(\bx_{0: T-1}; (\bx_T, \by)) =  F (\bx_{0: T-1}; (\bx_T, \by) ) - \bx_{0: T-1}$. 
Then, we can directly input the residual to the Anderson acceleration solver and obtain the final converged fixed point, \ie
\begin{align}
	\bx_{0: T-1}^* = \mathrm{RootSolve} \left( g(\bx_{0: T-1}; (\bx_T, \by)) \right),
\end{align}
where $\bx_0^*$ is our desired result at the end of sampling, and $\mathrm{RootSolve}(\cdot)$ is a fixed point solver using Anderson acceleration, which is implemented in Algorithm \ref{alg:anderson}.
For convenience, we define $g_k:=g(\bx_{0: T-1}^{(k)}; (\bx_T, \by))$.
Note that $\mathrm{RootSolve}(\cdot)$ can be implemented in the PyTorch package, and we use the same hyper-parameters as \cite{pokle2022deqddim}.
Moreover, Algorithm \ref{alg:anderson} is guaranteed to converge to a fixed point, which is verified in the experiment sections.
Note that we do not train all functions and diffusion models.

\begin{algorithm}[tb]
	\caption{Implementation of $\mathrm{RootSolve}(\cdot)$}
	\label{alg:anderson}
	\textbf{Require}: A degraded image $\by$, a pre-trained diffusion model, timesteps $T$, iterations $K$, an integer parameter $m\geq1$
	\begin{algorithmic}[1] 
		\State Initialize $\bx_T \sim\mathcal{N}(\mathbf{0},\mathbf{I}), \mathbf{x}_{\;i}^{(0)}=\bx_T, i=0,\dots,T{-}1$
        \State Calculate $\mathbf{x}_{0:T-1}^{(1)} = F\left(\mathbf{x}_{0:T-1}^{(0)}; (\bx_T, \by)\right)$ 
        \For{$k$ from $1$ to $K$}
        \State $m_k = \min \{m, k\}$
        \State $\bG_k = [g_{k-m_k}, \ldots, g_k]$
        \State Solve least-squares problem for $\balpha {=} [\alpha_0, \ldots, \alpha_{m_k}]$
        \State \qquad $ \balpha_k = \argmin\nolimits_{\balpha} \| \bG_k \balpha \|_2, s.t., \balpha^{\top} \1= 1$
        \State Update the sequence
        \State $\bx_{0:T-1}^{(k+1)} = \sum_{i=0}^{m_k} (\balpha_k)_i F\left(\bx_{0:T-1}^{(k-m_k+i)}; \left(\bx_T, \by\right)\right)$
        \EndFor
		\State \textbf{return} $\bx_{0}^*:=\bx_{0}^{K+1}$
	\end{algorithmic} \vspace{-1mm}
\end{algorithm}

Compared with most existing diffusion model-based IR methods \cite{chung2023dps,wang2022ddnm}, our method operating all states in parallel results in more accurate estimations of the intermediate latent states $\bx_t$, requiring fewer sampling steps.
It implies that we are able to obtain the better final sample $\bx_0^*$ based on these accurately estimated intermediate latent states $\bx_t$.

\begin{algorithm}[tb]
	\caption{Initialization Optimization via DEQ inversion}
	\label{alg:inversion}
	\textbf{Require}: A degraded image $\by$, a pre-trained diffusion model, update rate $\lambda$, total steps $S$.
	\begin{algorithmic}[1] 
		\State Initialize $\bx_{T}\sim\mathcal{N}(\0, \bI), \bx_{i}=\bx_T, i=0,\ldots, T{-}1$
		\For{steps from $1$ to $S$}
		\State Disable gradient computation, and compute $\bx_{0: T-1}^*$ according to Algorithm \ref{alg:anderson} \vspace{1mm}  
        \State \quad $\bx_{0: T-1}^* = \mathrm{RootSolve} \left( g(\bx_{0: T-1}; (\bx_T, \by)) \right)$ \vspace{1mm} 
		\State Enable gradient computation, and compute loss and use the 1-step grad $\partial \mL / \partial \bx_T$
        \State Update $\bx_T$ with a gradient descent: \vspace{1mm} 
        \State \qquad\qquad $\bx_T \leftarrow \bx_T + \lambda {\partial \mL} / {\partial \bx_T}$ \vspace{1mm}  
        \EndFor
		\State \textbf{return} $\bx_{T}$
	\end{algorithmic}
\end{algorithm}

\subsection{Initialization Optimization via DEQ Inversion}
Different initializations have diverse generations in some IR tasks, \eg colorization and inpainting. 
However, such diversity of generation is hard to control, and it is harmful to SR or deblurring which requires guaranteeing the identity.
To address this, we provide an interesting perspective to explore the initialization of our diffusion model. 

To achieve this, we first define a general loss function that can provide additional information. 
Specifically, given a degraded image $\by$ and the output of $\mathrm{RootSolve}$, \ie $\bx_0^*$, then the loss can be defined as
\begin{align}
    \mL = \ell \left( \phi(\bx_0^*), \varphi(\by) \right),
    \label{eqn:loss}
\end{align}
where $\ell$ can be $L_2$ loss or perceptual loss. 
For example, $\phi$ can be $\bA$ and $\varphi$ is an identity function; or $\phi$ is an identity function and $\varphi$ is a pre-trained IR model \cite{SWINIR,chen2023activating}. 
Based on the loss, we apply the implicit function theorem to compute the gradients of the loss $\mL$ \wrt $\bx_T$, \ie
\begin{equation}
    \!\!\dfrac{\partial \mathcal{L}}{\partial \bx_T} {=} {-} \dfrac{\partial \mathcal{L}}{\partial \bx^*_{0:T}} \left(J_{g}^{-1} \big|_{\bx^*_{0:T}}\right)  \dfrac{\partial F(\bx^*_{0:T-1}; (\bx_T, \by))}{\partial \bx_T},\!
    \label{eqn:deq-ift-backward}
\end{equation}
where $J_{g}^{-1} \big|_{\bx^*_{0:T}}$ is inverse Jacobian of $g(\bx_{0:T-1}; (\bx_T, \by))$ evaluated at $\bx_{0: T}^*$. In practical, we use an approximation version, \ie $\bM\approx J_{g}^{-1} |_{\bx^*_{0:T}}$, \eg 1-step gradient (\ie $\bM=\bI$)  \cite{fung2021fixed,geng2020attention,geng2021training}.
Note that the pre-trained diffusion model is frozen.
The gradients can be computed by using standard autograd packages in PyTorch.
Then, $\bx_T$ can be updated along the gradient, as shown in Algorithm \ref{alg:inversion}.

Different from existing diffusion model-based IR methods which have a large computational graph to store the gradients in the whole process, our method is more efficient due to the DEQ inversion.
In addition, with the help of the inversion method, our zero-shot IR methods can be extended to supervised learning by replacing the loss \eqref{eqn:loss} with $\mL=\| \bx_0^* - \bx_0 \|_F^2$ which we leave it in the future work.
\begin{table*}[t]
	\centering
    \resizebox{2.1\columnwidth}{!}{	
	\begin{tabular}{llccccccccccccccc}
		\toprule
		\multicolumn{1}{l}{\multirow{2}{*}{\!\!{Datasets}\!\!}} & \multirow{2}{*}{{Methods}} & \multicolumn{3}{c}{{2$\times$SR}} & \multicolumn{3}{c}{{4$\times$SR}} & \multicolumn{3}{c}{{Deblur (Gaussian)}} & \multicolumn{3}{c}{{Deblur (anisotropic)}} & {NFEs} \\ 
        \rule{0pt}{7pt} & & PSNR$\uparrow$&\!\!\!SSIM$\uparrow$\!\!\!&LPIPS$\downarrow$ & PSNR$\uparrow$&\!\!\!SSIM$\uparrow$\!\!\!&LPIPS$\downarrow$ &  PSNR$\uparrow$&\!\!\!SSIM$\uparrow$\!\!\!&LPIPS$\downarrow$ & PSNR$\uparrow$&\!\!\!SSIM$\uparrow$\!\!\!&LPIPS$\downarrow$ & {/Iters} \\
		\midrule
        \multirow{8}{*}{\!\!{ImageNet}\!\!}
        & Baseline 
        & 29.63 & \!\!\!\!0.875\!\!\!\! & 0.165 
        & 25.15 & \!\!\!\!0.699\!\!\!\! & 0.351 
        & 18.22 & \!\!\!\!0.529\!\!\!\! & 0.433 
        & 20.86 & \!\!\!\!0.544\!\!\!\! & 0.480  
        & - \\    
        & {DGP} \cite{pan2021dgp} 
        & 22.32 & \!\!\!\!0.583\!\!\!\! & 0.426 
        & 18.35 & \!\!\!\!0.398\!\!\!\! & 0.529 
        & 21.81 & \!\!\!\!0.522\!\!\!\! & 0.472 
        & 20.77 & \!\!\!\!0.459\!\!\!\! & 0.504 
        & 1500 \\
        & {DPS}  \cite{chung2023dps}  
        & 22.40 & \!\!\!\!0.597\!\!\!\! & 0.405 
        & 20.34 & \!\!\!\!0.488\!\!\!\! & 0.464 
        & 22.04 & \!\!\!\!0.569\!\!\!\! & 0.394 
        & 21.82 & \!\!\!\!0.561\!\!\!\! & 0.381 
        & 1000 \\
        & {ILVR} \cite{choi2021ilvr} 
        & 23.36 & \!\!\!\!0.613\!\!\!\! & 0.334
        & 22.76 & \!\!\!\!0.583\!\!\!\! & 0.383
        & - & - & - & - & - & - 
        & 100 \\
        & {DiffPIR} \cite{zhu2023diffpir} 
        & 27.16 & \!\!\!\!0.790\!\!\!\! & 0.214	
        & 24.31 & \!\!\!\!0.649\!\!\!\! & 0.350 
        & 25.32 & \!\!\!\!0.673\!\!\!\! & 0.296 
        & 23.37 & \!\!\!\!0.535\!\!\!\! & 0.439 
        & 100 \\ 
        & {DDRM} \cite{kawar2022ddrm} 
        & \rd 31.43 & \!\!\!\!\rd 0.906\!\!\!\! & \rd 0.117 
        & \rd 26.21 & \!\!\!\!\rd 0.745\!\!\!\! & \rd 0.288 
        & \rd 40.70 & \!\!\!\!\rd 0.978\!\!\!\! & \rd 0.040 
        & \rd 37.69 & \!\!\!\!\rd 0.964\!\!\!\! & \rd 0.057 
        & 20 \\
		& {DDNM} \cite{wang2022ddnm}  
        & \nd {31.81} & \!\!\!\!\nd {0.908}\!\!\!\! & \nd {0.097} 
        & \nd {26.49} & \!\!\!\!\nd {0.753}\!\!\!\! & \nd {0.266} 
        & \fs {43.83} & \!\!\!\!\fs {0.989}\!\!\!\! & \fs {0.018} 
        & \nd {38.40} & \!\!\!\!\nd {0.970}\!\!\!\! & \nd {0.038} 
        & 100 \\
		& {\textbf{DeqIR (Ours)}\!\!} 
        & \fs 32.35 & \fs \!\!\!\!0.913\!\!\!\! & \fs 0.082
        & \fs 27.47 & \fs \!\!\!\!0.781\!\!\!\! & \fs 0.230
        & \nd {43.42} & \!\!\!\!\nd{0.987}\!\!\!\! & \nd{0.021} 
        & \fs 39.47 & \!\!\!\!\fs{0.973}\!\!\!\! & \fs 0.036
        & {15} \\
		\midrule
        \multirow{8}{*}{\!\!{CelebA-HQ}\!\!\!}
        & {Baseline} 
        & 35.87 & \!\!\!\!0.953\!\!\!\! & 0.099
        & 30.12 & \!\!\!\!0.857\!\!\!\! & 0.240 
        & 18.94 & \!\!\!\!0.704\!\!\!\! & 0.337 
        & 23.16 & \!\!\!\!0.727\!\!\!\! & 0.354  
        & - \\ 
        & \rule{0pt}{7pt}{DGP} \cite{pan2021dgp} 
        & 28.61 & \!\!\!\!0.809\!\!\!\! & 0.279 
        & 25.25 & \!\!\!\!0.690\!\!\!\! & 0.405 
        & 27.02 & \!\!\!\!0.738\!\!\!\! & 0.372 
        & 25.73 & \!\!\!\!0.663\!\!\!\! & 0.426 
        & 1500 \\
        & \rule{0pt}{7pt}{DPS} \cite{chung2023dps}  
        & 28.71 & \!\!\!\!0.818\!\!\!\! & 0.219 
        & 25.01 & \!\!\!\!0.710\!\!\!\! & 0.282 
        & 27.56 & \!\!\!\!0.775\!\!\!\! & 0.229 
        & 26.91 & \!\!\!\!0.754\!\!\!\! & 0.234 
        & 1000 \\
        & \rule{0pt}{7pt}{ILVR} \cite{choi2021ilvr} 
        & 27.31 & \!\!\!\!0.783\!\!\!\! & 0.234
        & 27.09 & \!\!\!\!0.775\!\!\!\! & 0.245
        & - & - & - & - & - & - & 100 \\
        & \rule{0pt}{7pt}{DiffPIR} \cite{zhu2023diffpir} 
        & 32.51 & \!\!\!\!0.882\!\!\!\! & 0.156 
        & 28.60 & \!\!\!\!0.795\!\!\!\! & 0.228 
        & 30.63 & \!\!\!\!0.835\!\!\!\! & 0.197 
        & 29.32 & \!\!\!\!0.802\!\!\!\! & 0.232 
        & 100 \\ 
        & \rule{0pt}{7pt}{DDRM} \cite{kawar2022ddrm} 
        & \fs {36.76} & \!\!\!\!\nd {0.953}\!\!\!\! & \rd 0.074 
        & \nd {31.91} & \!\!\!\!\nd {0.880}\!\!\!\! & \nd {0.149} 
        & \rd 43.06 & \!\!\!\!\rd 0.983\!\!\!\! & \rd 0.036 
        & \rd 41.27 & \!\!\!\!\rd 0.976\!\!\!\! & \rd 0.053 
        & 20 \\
		& \rule{0pt}{7pt}{DDNM} \cite{wang2022ddnm}  
        & \rd 36.37 & \!\!\!\!\rd {0.950}\!\!\!\! & \nd {0.065} 
        & \rd 31.86 & \!\!\!\!\rd 0.876\!\!\!\! & \fs {0.136} 
        & \nd {46.99} & \!\!\!\!\nd {0.991}\!\!\!\! & \nd {0.021} 
        & \nd {43.43} & \!\!\!\!\nd {0.983}\!\!\!\! & \nd {0.037} 
        & 100 \\
		& \rule{0pt}{7pt}{\textbf{DeqIR (Ours)\!\!}} 
        & \nd {36.63} & \!\!\!\!\fs {0.954}\!\!\!\! & \fs {0.062} 
        & \fs {32.22} & \!\!\!\!\fs {0.889}\!\!\!\! & \rd 0.155 
        & \fs 47.18 & \fs \!\!\!\!0.992\!\!\!\! & \fs 0.019
        & \fs 43.57 & \fs \!\!\!\!0.984\!\!\!\! & \fs 0.036
        & {15} \\
		\bottomrule
	\end{tabular}
    }
    \vspace{-1.5mm}
	\caption{Quantitative results of zero-shot IR methods (including \textbf{super-resolution} and \textbf{deblurring}) on ImageNet and CelebA-HQ.  
    Best results are highlighted as \colorbox{colorFst}{\bf \!first\!}, \colorbox{colorSnd}{\!second\!} and \colorbox{colorTrd}{\!third\!}.}
	\label{tb:sr_deblur}
    \vspace{-4.5mm}
\end{table*}

\begin{figure*}[t]
  \centering
   \includegraphics[width=1\linewidth]{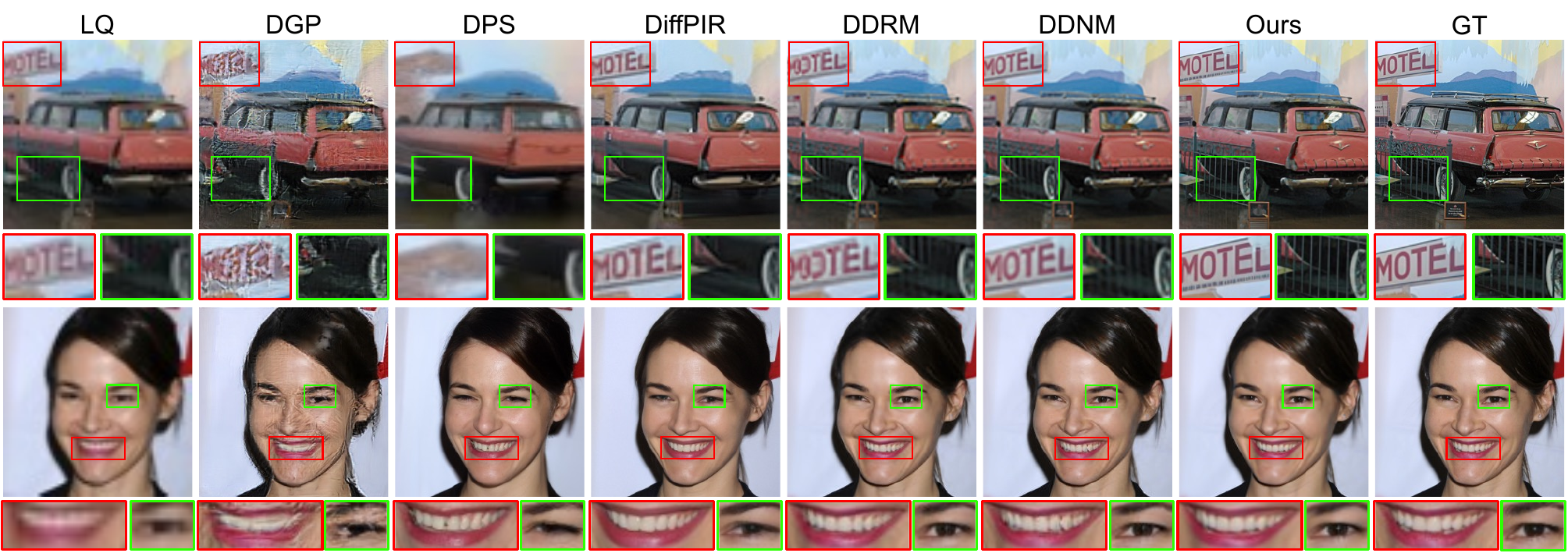}
   \vspace{-7mm}
   \caption{Qualitative results of zero-shot $4{\times}$ super-resolution methods on ImageNet and CelabA-HQ.}
   \vspace{-5mm}
   \label{fig:sr}
\end{figure*}

\section{Experiments}
\vspace{-1.5mm}
\noindent\textbf{Experiment settings.}
We conduct typical IR tasks, including SR, deblurring, colorization, and inpainting.
Specifically, we consider 2$\times$ and 4$\times$ bicubic downsampling for SR, Gaussian and anisotropic for deblurring, use an average grayscale operator in colorization, and use text and stripe masks in inpainting. 
For convenience, we choose ImageNet \cite{deng2009imagenet} and CelebA-HQ \cite{karras2018progressive} with 100 classes \cite{zhu2023diffpir} and the image size of $256{\times}256$ for validation, which have the same trend on 1k classes.
For fair comparisons, we use the same pre-trained diffusion models \cite{dhariwal2021diffusion} and \cite{lugmayr2022repaint} for ImageNet and CelebA-HQ, respectively. 
More details are put in Supplementary.

\begin{figure*}[t]
  \centering
   \includegraphics[width=1\linewidth]{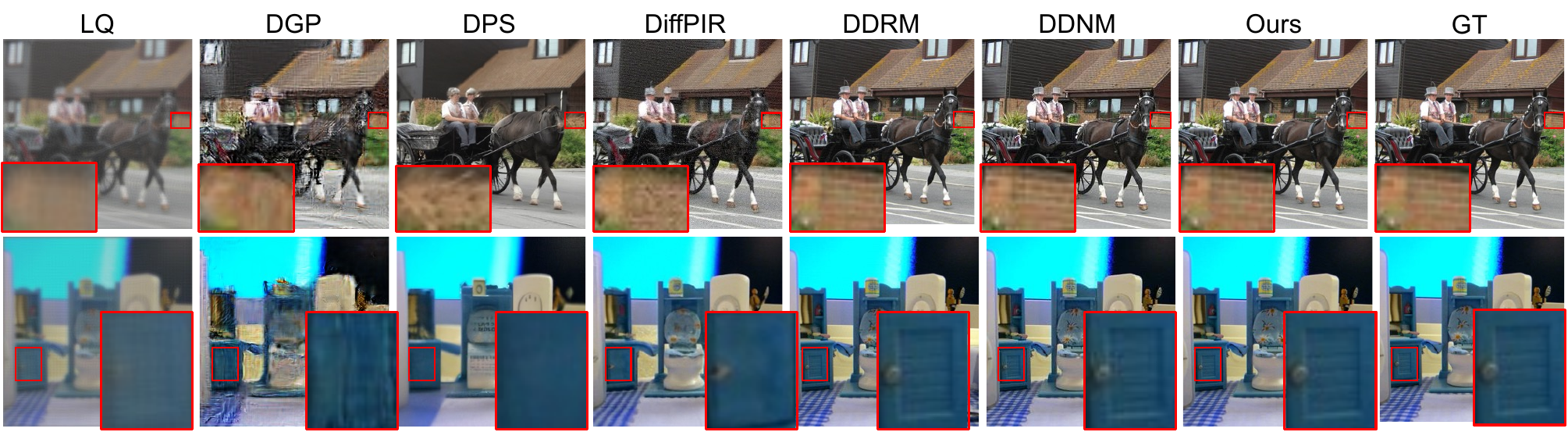}
   \vspace{-7mm}
   \caption{Qualitative results of zero-shot image deblurring (Gaussian) methods.}
   \vspace{-3mm}
   \label{fig:deblur}
\end{figure*}

\begin{table*}[t]
	\centering
    \begin{minipage}[t]{0.48\linewidth}
    \resizebox{1\columnwidth}{!}{	
	\begin{tabular}{lcccccc}
		\toprule
        \multirow{2}{*}{\!\!{Methods}} & \multicolumn{3}{c}{{Text mask}} & \multicolumn{3}{c}{{Stripe mask}} \\
        & PSNR$\uparrow$ &\!\!\!SSIM$\uparrow$\!\!\!&LPIPS$\downarrow$ & 
          PSNR$\uparrow$ &\!\!\!SSIM$\uparrow$\!\!\!&LPIPS$\downarrow$  \\
		\midrule
        \!\!Baseline
        & 14.55 & \!\!\!0.642\!\!\! & 0.515 
        & 9.02  & \!\!\!0.131\!\!\! & 0.730 \\    
        \rule{0pt}{7pt}{\!\!Palette \cite{saharia2022palette}} 
        & 38.09 & \!\!\!\rd {0.978}\!\!\! & \rd {0.027} 
        & 25.91 & \!\!\!0.733\!\!\! & 0.343 \\
        \rule{0pt}{7pt}{\!\!DDRM \cite{kawar2022ddrm}} 
        & 37.25 & \!\!\!0.969\!\!\! & 0.223 
        & 34.34 & \!\!\!0.933\!\!\! & 0.223 \\    
        \rule{0pt}{7pt}{\!\!RePaint \cite{lugmayr2022repaint}} 
        & \rd 38.54 & \!\!\!0.974\!\!\! & 0.039 
        & \rd {36.25} & \!\!\!\nd {0.951}\!\!\! & \nd {0.086} \\  
        \rule{0pt}{7pt}{\!\!DDNM \cite{wang2022ddnm}} 
        & \nd {39.45} & \!\!\!\nd {0.980}\!\!\! & \fs {0.023} 
        & \nd {36.75} & \!\!\!\fs {0.957}\!\!\! & \fs {0.076} \\    
		\rule{0pt}{7pt}{\!\!\textbf{DeqIR (Ours)}\!\!} 
        & \fs {39.72} & \!\!\!\fs {0.981}\!\!\! & \nd {0.026} 
        & \fs {36.99} & \!\!\!\rd {0.948}\!\!\! & \rd {0.091} \\
		\bottomrule
	\end{tabular}
    }
    \vspace{-1mm}
	\caption{Comparisons of zero-shot \textbf{inpainting} methods on CelebA. }
    \vspace{-5mm}
	\label{tb:inpaint}
    \end{minipage} 
    \hfill
    \begin{minipage}[t]{0.48\linewidth}
	\centering
    \resizebox{1\columnwidth}{!}{	
	\begin{tabular}{lcccccc}
		\toprule
        \multirow{2}{*}{{\!\!Methods}} & \multicolumn{3}{c}{{ImageNet}} & \multicolumn{3}{c}{{CelebA-HQ}} \\
        & Cons$\downarrow$& \!\!\!\!LPIPS$\downarrow$\!\!\!\! & FID$\downarrow$ & Cons$\downarrow$ & \!\!\!\!LPIPS$\downarrow$\!\!\!\! & FID$\downarrow$  \\
		\midrule 
        \!\!{Baseline} 
        & 0 & \!\!\!\!0.196\!\!\!\! & 90.93 
        & 0 & \!\!\!\!0.210\!\!\!\! & 70.69 \\
        \rule{0pt}{10pt}{\!\!DGP \cite{pan2021dgp}} 
        & -	& \!\!\!\!0.256\!\!\!\! & 99.86 
        & - & \!\!\!\!0.218\!\!\!\!  & 73.24 \\
        \rule{0pt}{10pt}{\!\!DDRM \cite{kawar2022ddrm}} 
        & 265.08 & \!\!\!\!0.223\!\!\!\! & 79.42 
        & 472.25 & \!\!\!\!0.245\!\!\!\! & 57.29 \\    
        \rule{0pt}{10pt}{\!\!DDNM \cite{wang2022ddnm}} 
        & \nd 45.07	& \!\!\!\!\nd 0.186\!\!\!\! & \nd 77.21 
        & \nd 51.43	& \!\!\!\!\nd 0.139\!\!\!\! & \nd 45.73 \\    
		\rule{0pt}{10pt}{\!\!\textbf{DeqIR (Ours)}} 
        & \fs 43.15	& \!\!\!\!\fs 0.171\!\!\!\! & \fs 70.94 
        & \fs 50.16 & \!\!\!\!\fs 0.092\!\!\!\! & \fs 43.98 \\
		\bottomrule
	\end{tabular}
    }
    \vspace{-1mm}
	\caption{Quantitative results of zero-shot \textbf{colorization} methods. }
    \vspace{-5mm}
	\label{tb:color}
    \end{minipage}
\end{table*}

\vspace{1mm}
\noindent\textbf{Evaluation metrics.}
We use PSNR, SSIM and LPIPS as the evaluation metrics for most IR tasks.
For the task of colorization, we use the Consistency metric \cite{wang2022ddnm} and FID because PSNR and SSIM cannot reflect the performance \cite{wang2022ddnm}.
In general, higher PSNR and SSIM, and lower LPIPS and FID mean better performance.
In addition, we report the number of NFEs (timesteps) or iterations for each method.

\subsection{Evaluation on Image Super-Resolution}
\vspace{-1.5mm}
We compare our method with a GAN-based IR method (\eg DGP) and SOTA zero-shot diffusion model-based IR methods (\eg DPS \cite{chung2023dps}, DiffPIR \cite{zhu2023diffpir}, DDRM \cite{kawar2022ddrm} and DDNM \cite{wang2022ddnm}) on ImageNet and CelebA-HQ datasets. 
In addition, we use the bicubic upscaling as a baseline for SR.

In Table \ref{tb:sr_deblur}, our method outperforms most methods under different metrics on both ImageNet and CelebA-HQ.
In particular, compared with the competitive IR method DDNM, our method on ImageNet surpasses it by an LPIPS margin of up to 0.036, and by a PSNR margin of up to 0.98dB.
Moreover, our method only needs 15 iteration steps, compared with DDNM (100 steps).
We provide more details and quantitative results of other scales in Supplementary.

For the qualitative results, our method achieves the best visual quality containing more realistic textures, as shown in Figure \ref{fig:sr}. 
These visual comparisons align with the quantitative results, demonstrating the effectiveness of our method.
More visual results are put in Supplementary Materials.

\subsection{Evaluation on Image Deblurring}
\vspace{-1.5mm}
We compare the same zero-shot IR methods used in the SR task.
In addition, we use $\bA^{\dag} \by$ as a baseline. 
In this experiment, we mainly consider Gaussian and anisotropic kernels to evaluate the performance of all models.

In Table \ref{tb:sr_deblur}, the quantitative results show that our method achieves the best performance on all datasets, except for Gaussian deblurring on ImageNet. 
Compared with DDNM \cite{wang2022ddnm}, the PSNR improvement of our method can be up to 1.07dB for anisotropic deblurring.
In Figure \ref{fig:deblur}, our generated images have the best visual quality with more realistic details which are close to GT images. 
We provide more quantitative and qualitative results (including more kernels) in Supplementary Materials.

\subsection{Evaluation on Image Inpainting}
\vspace{-1mm}
For the image inpainting task, we compare our method with SOTA inpainting methods, including Palette \cite{saharia2022palette}, RePaint \cite{lugmayr2022repaint}, DDRM \cite{kawar2022ddrm} and DDNM \cite{wang2022ddnm}. 
We also use $\bA^{\dag} \by$ as a baseline.
In addition, we consider the text mask and stripe mask as examples and show the results on CelebA-HQ in Table \ref{tb:inpaint}. 
The results of more masks and results on ImageNet are put in Supplementary Materials.

In Table \ref{tb:inpaint}, our method outperforms Palette \cite{saharia2022palette} and DDRM \cite{kawar2022ddrm} significantly, and has comparable performance with RePaint \cite{lugmayr2022repaint} and DDNM \cite{wang2022ddnm}. 
In Figure \ref{fig:inpainting}, taking the ``mouth'' in the generated face images as an example, our method generates clear structures and details that are not only more realistic but also more reasonable compared to other inpainting methods. 
In contrast, other methods may introduce blur artifacts.

\begin{figure}[h]
  \centering
  \vspace{-2mm}
   \includegraphics[width=1\linewidth]{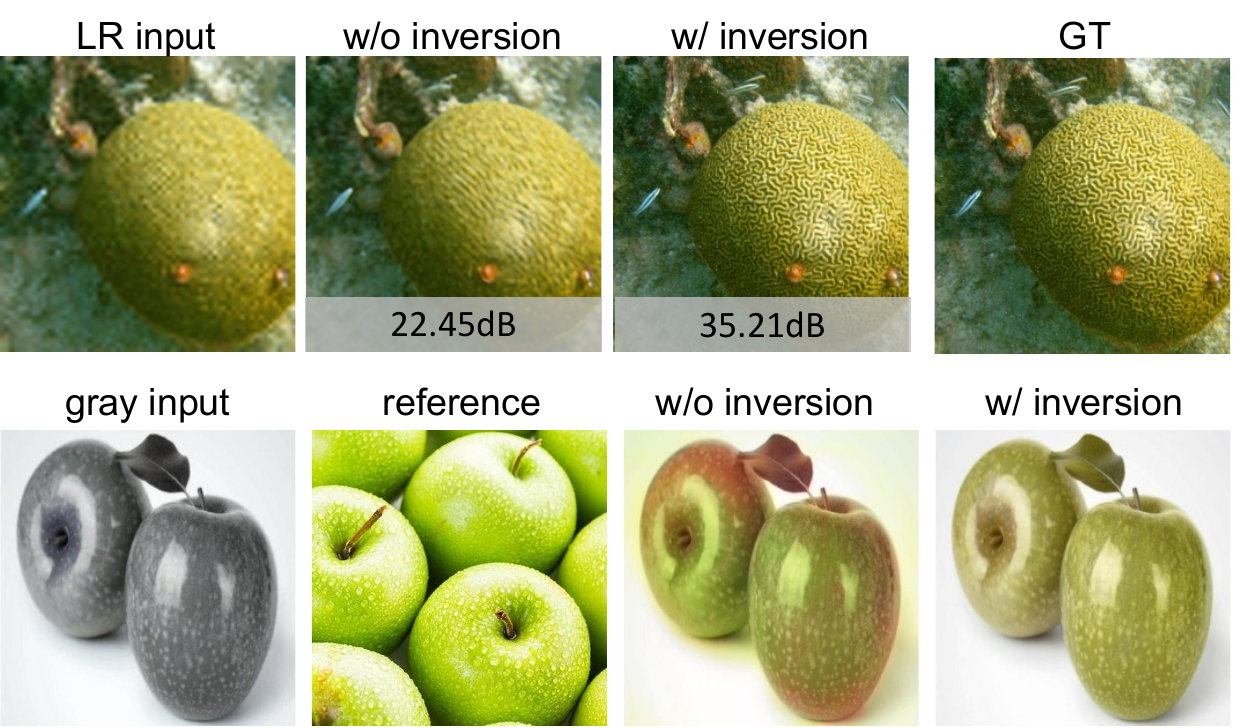}
   \vspace{-7mm}
   \caption{Interesting applications of DEQ inversion.}
   \vspace{-4mm}
   \label{fig:inversion}
\end{figure}

\subsection{Evaluation on Image Colorization}
\vspace{-1.5mm}
We compare our method with SOTA methods (\ie DGP \cite{pan2021dgp}, DDRM \cite{kawar2022ddrm} and DDNM \cite{wang2022ddnm}).
We also use $\bA^{\dag} \by$ as a baseline.
In addition to LPIPS, we additionally use the Consistency metric and FID to evaluate the image quality.

In Table \ref{tb:color}, our method achieves the best performance on both ImageNet and CelebA-HQ under different metrics. 
As shown in Figure \ref{fig:colorization}, our method restores images with reasonable color.
In contrast, other methods may restore part of the color (as observed in the ``tree'') or unreasonable color (\eg evident in the ``building'' in DGP \cite{pan2021dgp}).

\begin{figure*}[t]
  \centering
   \vspace{-3mm}
   \includegraphics[width=1\linewidth]{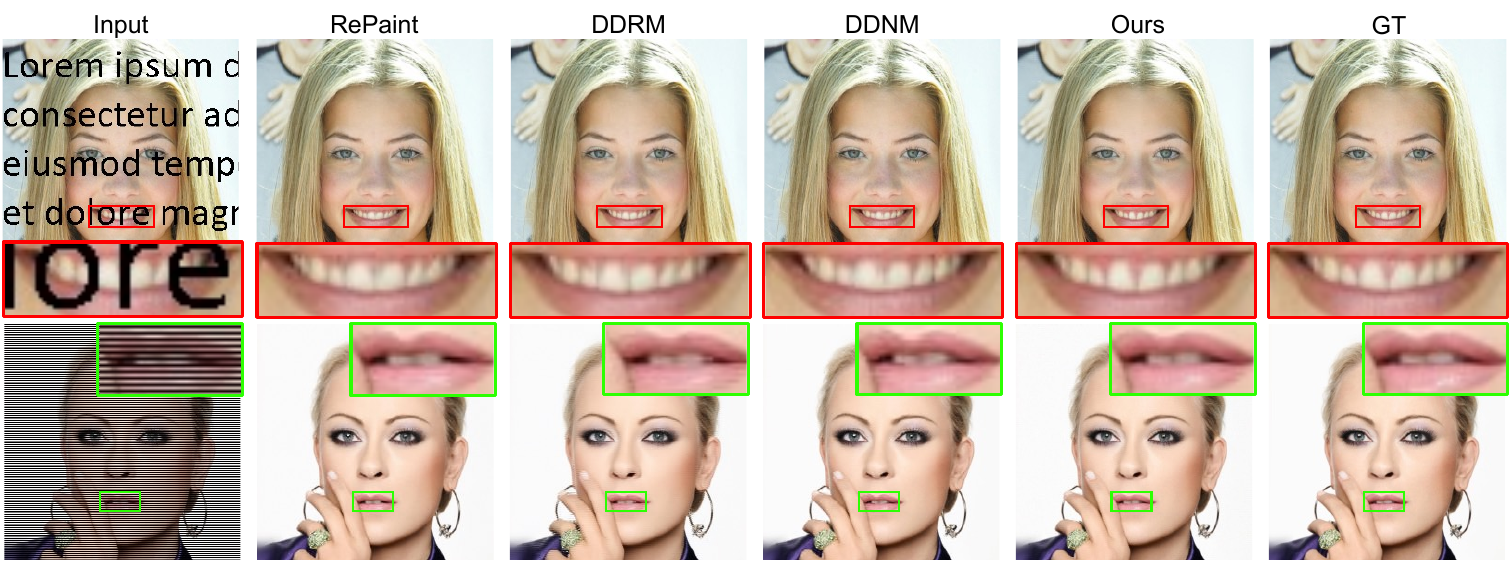}
   \vspace{-7mm}
   \caption{Qualitative results of image inpainting methods on CelebA-HQ.}
   \vspace{-4mm}
   \label{fig:inpainting}
\end{figure*}

\begin{figure*}[t]
  \centering
   \vspace{-0.5mm}
   \includegraphics[width=1\linewidth]{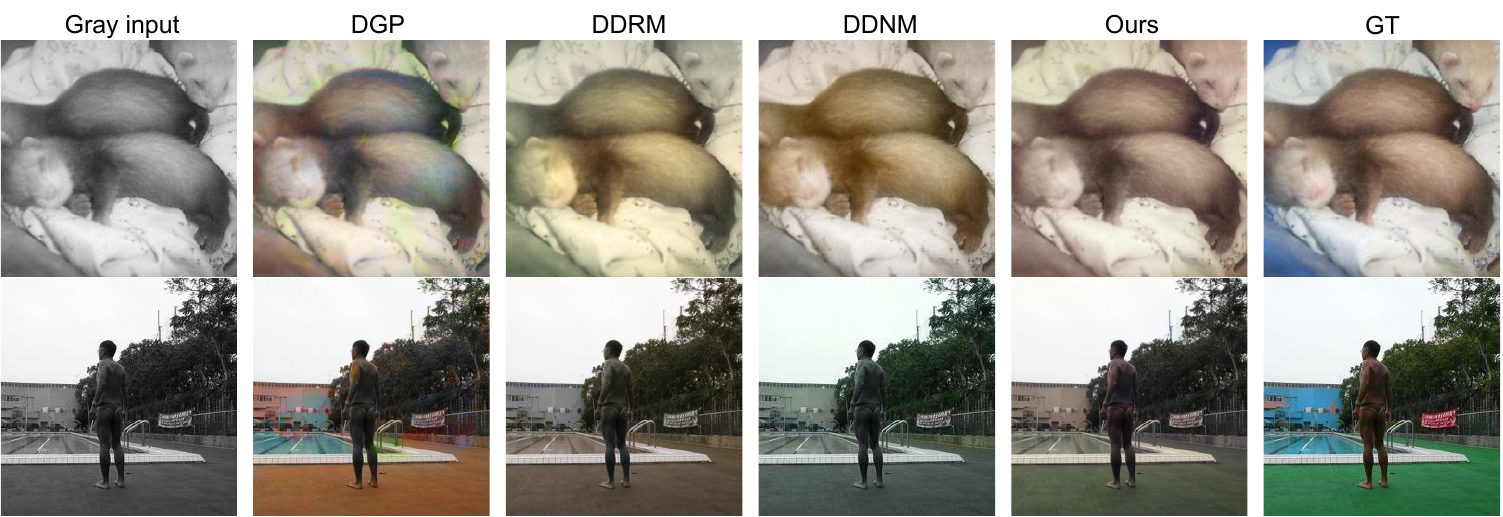}
   \vspace{-7mm}
   \caption{Qualitative results of image colorization methods on ImageNet.}
   \vspace{-5mm}
   \label{fig:colorization}
\end{figure*}

\subsection{Evaluation on DEQ Inversion}
\vspace{-1.5mm}
We extend our method using DEQ inversion to interesting applications, \eg SR with optimized initialization (top) and reference-based colorization (bottom), as shown in Figure \ref{fig:inversion}.
We found that optimizing the initialization is able to improve PSNR and control the generation in the desired direction.
More details and results are put in Supplementary.

\begin{figure*}[t]
	\centering
	\includegraphics[width=1\linewidth]{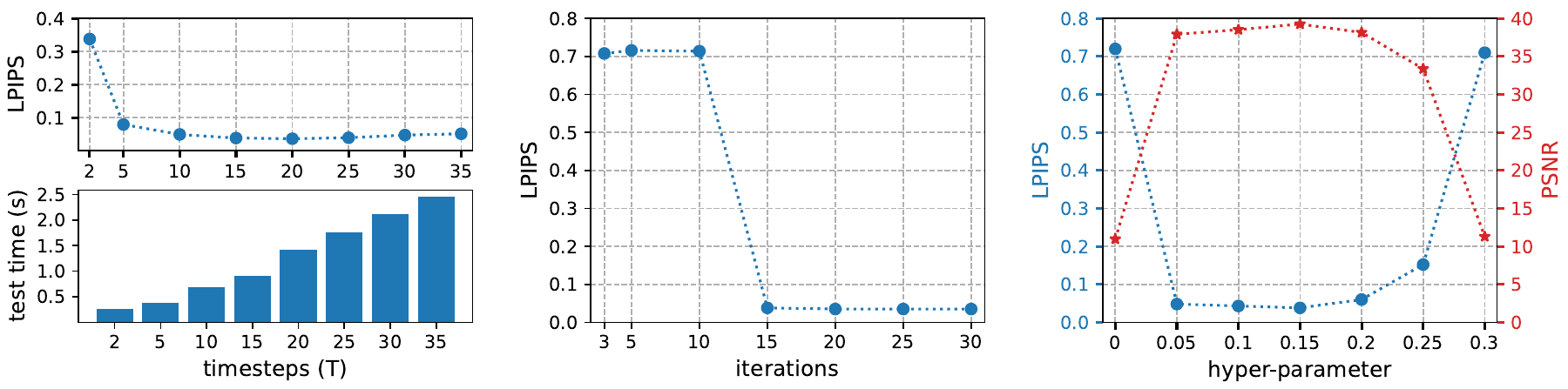}
	\vspace{-8mm}
	\caption{Ablation study of timesteps (left), iteration (middle) and hyper-parameters (right) for anisotropic deblurring on ImageNet.}
	\vspace{-7mm}
	\label{fig:ablation}
\end{figure*}

\begin{figure}[t]
	\centering
	\includegraphics[width=1\linewidth]{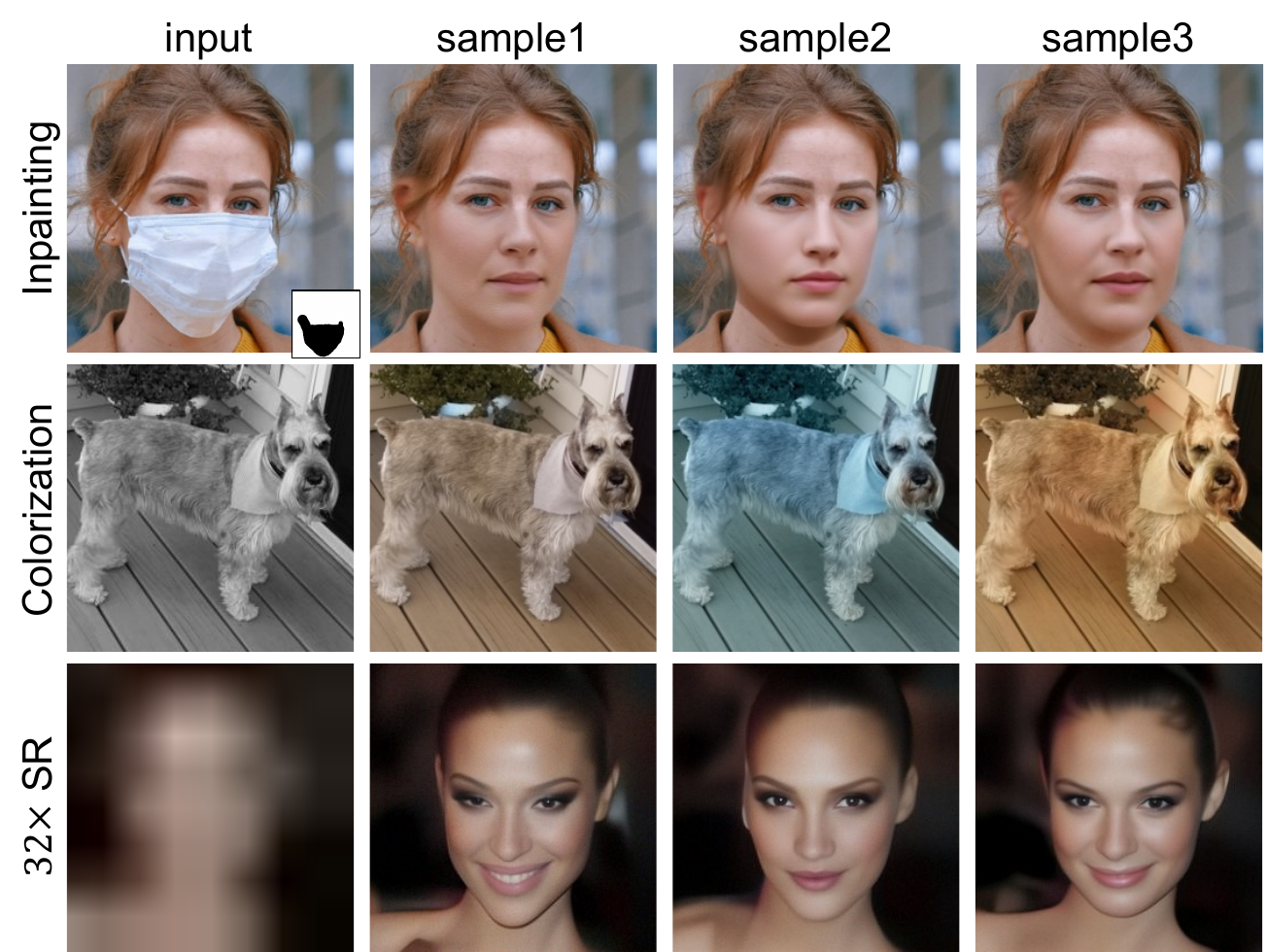}
	\vspace{-7mm}
	\caption{Diversity of generation of our method.}
	\label{fig:diversity}
	\vspace{-6mm}
\end{figure}

\subsection{Ablation Study}
\vspace{-1mm}
\noindent\textbf{Effect of timesteps.}
We study the impact of timesteps in our diffusion models.
Specifically, we change the number of timesteps from 2 to 35. 
In Figure \ref{fig:ablation} (left), the image quality improves with additional timesteps until it stabilizes.
However, more timesteps lead to a larger memory and slower convergence.
To trade off between performance and efficiency, we set the timesteps to 20 in this experiment.

\vspace{1mm}
\noindent\textbf{Effect of iterations.}
We investigate the impact of varying the number of iterations in the Anderson acceleration in Figure \ref{fig:ablation} (middle).
Increasing the number of iterations results in improved performance.
As we can see, 15 iterations are sufficient to converge to satisfactory results.

\vspace{1mm}
\noindent\textbf{Effect of hyper-parameter $\eta$.}
We further investigate the influence of the hyper-parameter $\eta$ in our proposed analytic formulation, \ie Eqn. \eqref{eqn:deqir-xTk}.
In Figure \ref{fig:ablation} (right), different values of the hyper-parameter have different effects on the performance.
Larger values introduce more noise in the generated image, while smaller values may limit the restoration performance. 
Therefore, we set the hyper-parameter $\eta$ to 0.15 in this task.

\subsection{Diversity of Generation}
\vspace{-1.5mm}
To investigate the ability of our method, we show diverse results for different tasks in Figure \ref{fig:diversity}.
With different seeds, our method is able to generate diverse images with realistic details on inpainting and colorization.
For $32{\times}$ SR, the input face image is severely degraded, and the generated faces are realistic but they are difficult to retain the identity.

\subsection{Real-World Applications}
\vspace{-1.5mm}
Our method can be applied in real-world settings which may have unknown, non-linear and complex degradations.

\vspace{1mm}
\noindent\textbf{Old photo restoration.}
The degradations in old photo restoration suffer from non-linear and unknown artifacts. 
Such artifacts are often covered by a hand-drawn mask (denoted by $\bA_{\text{mask}}$).
The degradation can be a composite of $\bA_{\text{mask}}$ and a colorization degradation (denoted by $\bA_{\text{color}}$), and its pseudo-inverse can also be constructed by hand.
In Figure \ref{fig:old_cor_real_sr} (top), our method achieves a remarkable enhancement with facial details, effectively reducing the visible artifacts while preserving finer details.
The inpainting and colorization results serve as a compelling illustration of the effectiveness of our old photo restoration technique.

\noindent\textbf{Real-world SR.}
Real-world degradations may have non-Gaussian noise, unknown compression noise and downscaling.
We use a restoration model \cite{restormer} to provide the prior information to the input noise.
As shown in Figure \ref{fig:old_cor_real_sr} (bottom), our method achieves good robustness to the real noise. 
Notably, our method successfully preserves the facial identity and produces realistic results with rich details.

\vspace{1mm}
\noindent\textbf{Arbitrary size.}
Our method can also be used in images with arbitrary sizes.
Similarly to \cite{wang2022ddnm,SWINIR}, we crop a large-size image as multiple overlapped patches and then test each patch. 
Last we concatenate the generation as the final results.
We put the results in Supplementary due to the limited space.

\subsection{Further Experiments}
\vspace{-1.5mm}
\noindent\textbf{Running time.}
We compare the running time of different methods for anisotropic deblurring on ImageNet. 
For fair comparisons, we evaluate all methods on $256{\times}256$ input images on NVIDIA TITAN RTX using their publicly available code. 
In Table \ref{tb:runtime}, our method with 10 steps has a comparable running time to DDNM \cite{wang2022ddnm}.
DDRM \cite{kawar2022ddrm} with 20 steps is faster than our method, but it is worse than our method.

\vspace{1mm}
\noindent\textbf{Comparisons with supervised learning.}
We compare our zero-shot method with supervised learning methods in Table \ref{tb:supervised}.
Our method outperforms GAN-based methods and LDM \cite{rombach2022ldm}, but it is worse than DiffIR \cite{xia2023diffir}.
However, these methods have limited generalization on other tasks.

\begin{figure}[t]
	\centering
	\includegraphics[width=1\linewidth]{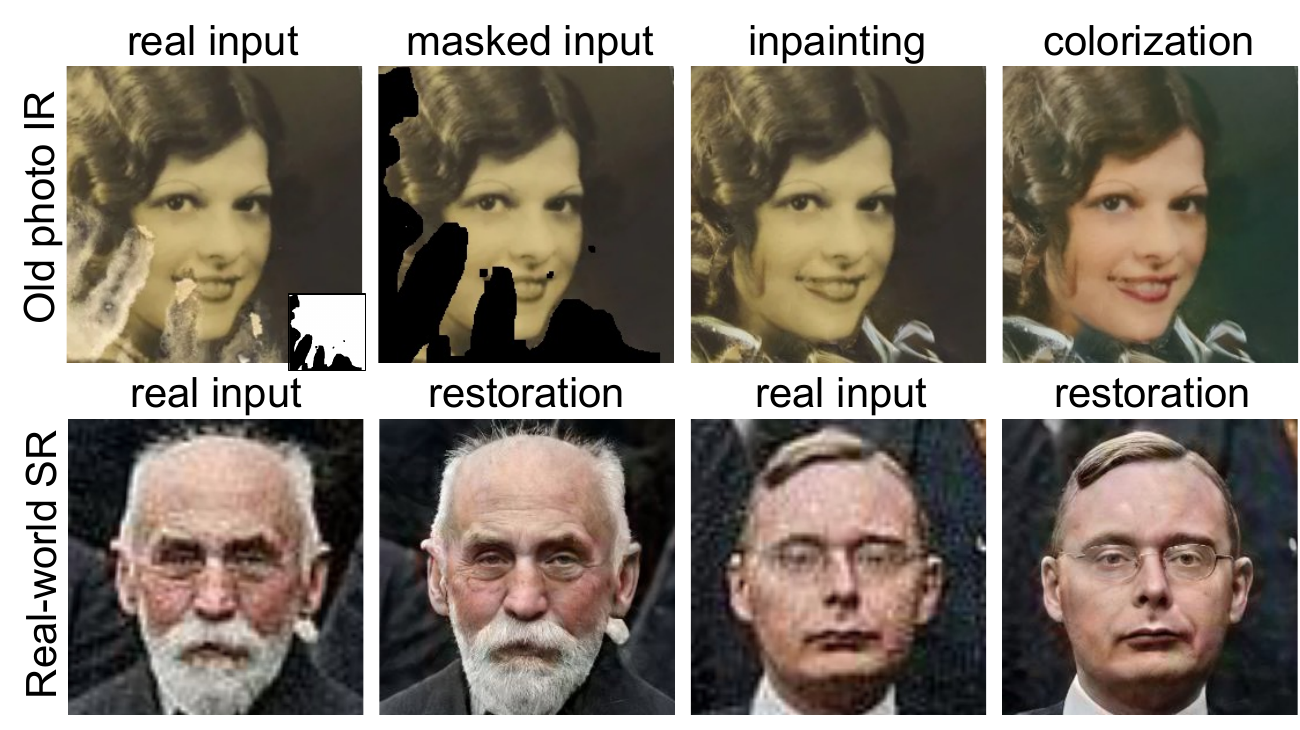}
	\vspace{-8mm}
	\caption{Real-world applications of our method.}
	\vspace{-3mm}
	\label{fig:old_cor_real_sr}
\end{figure}

\begin{table}[t]
	\centering
	\resizebox{1\columnwidth}{!}{	
		\begin{tabular}{lccccccc}
			\toprule
			Methods\!\!\!\! & \!\!DPS \cite{chung2023dps}\!\! & \!\!DDRM \cite{kawar2022ddrm}\!\! & \!\!DDNM \cite{wang2022ddnm}\!\! & \!\!Ours-10\!\! & \!\!Ours-15\!\! & \!\!Ours-20\!\! \\	
			\midrule  
			{{Time (s)}} 
			& 468.85 & \fs 5.26 & \rd 12.67 & \nd 12.11 & 16.53 & 21.19 \\
			{{PSNR$\uparrow$}} 
			& 21.82 & 37.69 & 38.40 & \rd 38.58 & \nd 39.21 & \fs 39.47 \\
			\bottomrule
		\end{tabular}
	}
	\vspace{-1mm}
	\caption{Running time of different methods. $*$-T: T timesteps.
	}
	\vspace{-3mm}
	\label{tb:runtime}
\end{table}

\begin{table}[t]
	\centering
	\resizebox{1\columnwidth}{!}{	
		\begin{tabular}{lccccccccccc}
			\toprule
			\multirow{2}{*}{{Methods}} & \multicolumn{3}{c}{{ImageNet}} & \multicolumn{3}{c}{{CelebaA-HQ}} \\
			& PSNR$\uparrow$ & \!\!\!SSIM$\uparrow$\!\!\! & LPIPS$\downarrow$ & PSNR$\uparrow$ & \!\!\!SSIM$\uparrow$\!\!\! & LPIPS$\downarrow$   \\
			\midrule
			{SRGAN \cite{ledig2017srgan}} 
			& 24.83 & \!\!\!0.696\!\!\! & \nd 0.245 
			& 31.16 & \!\!\!0.868\!\!\! & 0.164 \\  
			{BSRGAN \cite{zhang2021bsrgan}\!\!\!\!} 
			& 23.65 & \!\!\!0.651\!\!\! & 0.331 
			& 27.80 & \!\!\!0.808\!\!\! & 0.216 \\  
			{LDM \cite{rombach2022ldm}} 
			& 22.34 & \!\!\!0.606\!\!\! & 0.318 
			& 27.18 & \!\!\!0.783\!\!\! & 0.208 \\
			{DiffIR \cite{xia2023diffir}} 
			& \fs 29.25 & \!\!\!\fs 0.814\!\!\! & \fs 0.235 
			& \fs 34.96 & \!\!\!\fs 0.924\!\!\! & \fs 0.121 \\
			{\textbf{DeqIR (Ours)}} 
			& \nd 27.44 & \!\!\!\nd 0.782\!\!\! & \fs 0.235 
			& \nd 32.19 & \!\!\!\nd 0.887\!\!\! & \nd 0.154 \\
			\bottomrule
		\end{tabular}
	}
	\vspace{-1mm}
	\caption{Comparisons of supervised learning methods and our zero-shot method on ImageNet for $4\times$ SR. }
	\vspace{-5mm}
	\label{tb:supervised}
\end{table}

\vspace{-2mm}
\section{Conclusion}
\vspace{-1.5mm}
In this paper, we have proposed a novel zero-shot diffusion model-based IR method, called DeqIR.
Specifically, we model diffusion model-based IR generation as a deep equilibrium (DEQ) fixed point system.
Our IR method can conduct parallel sampling, instead of long sequential sampling in traditional diffusion models.
Based on the DEQ inversion, we are able to explore the relationship between the restoration and initialization. 
With the initialization optimization, the restoration performance can be improved and the generation direction can be guided with additional information. 
Extensive experiments demonstrate that our proposed DeqIR achieves better performance on different IR tasks.
Moreover, our DeqIR can be generalized to real-world applications.

\begin{flushleft}
	\textbf{Acknowledgments.}
	This work was partly supported by The Alexander von Humboldt Foundation.
\end{flushleft}
{
    \small
    \bibliographystyle{ieeenat_fullname}
    \bibliography{main}

\begin{thebibliography}{87}
\providecommand{\natexlab}[1]{#1}
\providecommand{\url}[1]{\texttt{#1}}
\expandafter\ifx\csname urlstyle\endcsname\relax
  \providecommand{\doi}[1]{doi: #1}\else
  \providecommand{\doi}{doi: \begingroup \urlstyle{rm}\Url}\fi

\bibitem[Amos and Kolter(2017)]{amos2017optnet}
Brandon Amos and J~Zico Kolter.
\newblock Optnet: Differentiable optimization as a layer in neural networks.
\newblock In \emph{ICML}, 2017.

\bibitem[Anderson(1965)]{anderson1965iterative}
Donald~G Anderson.
\newblock Iterative procedures for nonlinear integral equations.
\newblock \emph{Journal of the ACM}, 1965.

\bibitem[Bai et~al.(2019)Bai, Kolter, and Koltun]{bai2019deq}
Shaojie Bai, J~Zico Kolter, and Vladlen Koltun.
\newblock Deep equilibrium models.
\newblock In \emph{NeurIPS}, 2019.

\bibitem[Bai et~al.(2020)Bai, Koltun, and Kolter]{bai2020mdeq}
Shaojie Bai, Vladlen Koltun, and J~Zico Kolter.
\newblock Multiscale deep equilibrium models.
\newblock In \emph{NeurIPS}, 2020.

\bibitem[Bai et~al.(2022)Bai, Geng, Savani, and Kolter]{bai2022deep}
Shaojie Bai, Zhengyang Geng, Yash Savani, and J~Zico Kolter.
\newblock Deep equilibrium optical flow estimation.
\newblock In \emph{CVPR}, 2022.

\bibitem[Broyden(1965)]{broyden1965class}
Charles~G Broyden.
\newblock A class of methods for solving nonlinear simultaneous equations.
\newblock \emph{Mathematics of computation}, 1965.

\bibitem[Cao et~al.(2018)Cao, Guo, Wu, Shen, Huang, and Tan]{cao2018lccgan}
Jiezhang Cao, Yong Guo, Qingyao Wu, Chunhua Shen, Junzhou Huang, and Mingkui
  Tan.
\newblock Adversarial learning with local coordinate coding.
\newblock In \emph{ICML}, 2018.

\bibitem[Cao et~al.(2019)Cao, Mo, Zhang, Jia, Shen, and Tan]{cao2019mwgan}
Jiezhang Cao, Langyuan Mo, Yifan Zhang, Kui Jia, Chunhua Shen, and Mingkui Tan.
\newblock Multi-marginal wasserstein gan.
\newblock In \emph{NeurIPS}, 2019.

\bibitem[Cao et~al.(2020)Cao, Guo, Wu, Shen, Huang, and Tan]{cao2020improving}
Jiezhang Cao, Yong Guo, Qingyao Wu, Chunhua Shen, Junzhou Huang, and Mingkui
  Tan.
\newblock Improving generative adversarial networks with local coordinate
  coding.
\newblock \emph{TPAMI}, 2020.

\bibitem[Cao et~al.(2021)Cao, Li, Zhang, and Van~Gool]{cao2021vsrt}
Jiezhang Cao, Yawei Li, Kai Zhang, and Luc Van~Gool.
\newblock Video super-resolution transformer.
\newblock \emph{arXiv preprint arXiv:2106.06847}, 2021.

\bibitem[Cao et~al.(2022{\natexlab{a}})Cao, Liang, Zhang, Li, Zhang, Wang, and
  Gool]{cao2022datsr}
Jiezhang Cao, Jingyun Liang, Kai Zhang, Yawei Li, Yulun Zhang, Wenguan Wang,
  and Luc~Van Gool.
\newblock Reference-based image super-resolution with deformable attention
  transformer.
\newblock In \emph{ECCV}, 2022{\natexlab{a}}.

\bibitem[Cao et~al.(2022{\natexlab{b}})Cao, Liang, Zhang, Wang, Wang, Zhang,
  Tang, and Van~Gool]{cao2022davsr}
Jiezhang Cao, Jingyun Liang, Kai Zhang, Wenguan Wang, Qin Wang, Yulun Zhang,
  Hao Tang, and Luc Van~Gool.
\newblock Towards interpretable video super-resolution via alternating
  optimization.
\newblock In \emph{ECCV}, 2022{\natexlab{b}}.

\bibitem[Cao et~al.(2023)Cao, Wang, Xian, Li, Ni, Pi, Zhang, Zhang, Timofte,
  and Van~Gool]{cao2023ciaosr}
Jiezhang Cao, Qin Wang, Yongqin Xian, Yawei Li, Bingbing Ni, Zhiming Pi, Kai
  Zhang, Yulun Zhang, Radu Timofte, and Luc Van~Gool.
\newblock Ciaosr: Continuous implicit attention-in-attention network for
  arbitrary-scale image super-resolution.
\newblock In \emph{CVPR}, 2023.

\bibitem[Cavigelli et~al.(2017)Cavigelli, Hager, and Benini]{CAS-CNN}
Lukas Cavigelli, Pascal Hager, and Luca Benini.
\newblock Cas-cnn: A deep convolutional neural network for image compression
  artifact suppression.
\newblock In \emph{IJCNN}, 2017.

\bibitem[Chen et~al.(2021)Chen, Wang, Guo, Xu, Deng, Liu, Ma, Xu, Xu, and
  Gao]{IPT}
Hanting Chen, Yunhe Wang, Tianyu Guo, Chang Xu, Yiping Deng, Zhenhua Liu, Siwei
  Ma, Chunjing Xu, Chao Xu, and Wen Gao.
\newblock Pre-trained image processing transformer.
\newblock In \emph{CVPR}, 2021.

\bibitem[Chen et~al.(2022)Chen, Chu, Zhang, and Sun]{NAFNet}
Liangyu Chen, Xiaojie Chu, Xiangyu Zhang, and Jian Sun.
\newblock Simple baselines for image restoration.
\newblock In \emph{ECCV}, 2022.

\bibitem[Chen et~al.(2018)Chen, Rubanova, Bettencourt, and
  Duvenaud]{chen2018neural}
Ricky~TQ Chen, Yulia Rubanova, Jesse Bettencourt, and David~K Duvenaud.
\newblock Neural ordinary differential equations.
\newblock In \emph{NeurIPS}, 2018.

\bibitem[Chen et~al.(2023)Chen, Wang, Zhou, Qiao, and Dong]{chen2023activating}
Xiangyu Chen, Xintao Wang, Jiantao Zhou, Yu Qiao, and Chao Dong.
\newblock Activating more pixels in image super-resolution transformer.
\newblock In \emph{CVPR}, 2023.

\bibitem[Choi et~al.(2021)Choi, Kim, Jeong, Gwon, and Yoon]{choi2021ilvr}
Jooyoung Choi, Sungwon Kim, Yonghyun Jeong, Youngjune Gwon, and Sungroh Yoon.
\newblock Ilvr: Conditioning method for denoising diffusion probabilistic
  models.
\newblock In \emph{ICCV}, 2021.

\bibitem[Chung et~al.(2023)Chung, Kim, Mccann, Klasky, and Ye]{chung2023dps}
Hyungjin Chung, Jeongsol Kim, Michael~T Mccann, Marc~L Klasky, and Jong~Chul
  Ye.
\newblock Diffusion posterior sampling for general noisy inverse problems.
\newblock In \emph{ICLR}, 2023.

\bibitem[Dai et~al.(2019)Dai, Cai, Zhang, Xia, and Zhang]{SAN}
Tao Dai, Jianrui Cai, Yongbing Zhang, Shu-Tao Xia, and Lei Zhang.
\newblock Second-order attention network for single image super-resolution.
\newblock In \emph{CVPR}, 2019.

\bibitem[Deng et~al.(2009)Deng, Dong, Socher, Li, Li, and
  Fei-Fei]{deng2009imagenet}
Jia Deng, Wei Dong, Richard Socher, Li-Jia Li, Kai Li, and Li Fei-Fei.
\newblock Imagenet: A large-scale hierarchical image database.
\newblock In \emph{CVPR}, 2009.

\bibitem[Dhariwal and Nichol(2021)]{dhariwal2021diffusion}
Prafulla Dhariwal and Alexander Nichol.
\newblock Diffusion models beat gans on image synthesis.
\newblock In \emph{NeurIPS}, 2021.

\bibitem[Djolonga and Krause(2017)]{djolonga2017differentiable}
Josip Djolonga and Andreas Krause.
\newblock Differentiable learning of submodular models.
\newblock In \emph{NeurIPS}, 2017.

\bibitem[Dong et~al.(2015{\natexlab{a}})Dong, Deng, Loy, and
  Tang]{dong201arcnn}
Chao Dong, Yubin Deng, Chen~Change Loy, and Xiaoou Tang.
\newblock Compression artifacts reduction by a deep convolutional network.
\newblock In \emph{ICCV}, 2015{\natexlab{a}}.

\bibitem[Dong et~al.(2015{\natexlab{b}})Dong, Loy, He, and Tang]{dong2015srcnn}
Chao Dong, Chen~Change Loy, Kaiming He, and Xiaoou Tang.
\newblock Image super-resolution using deep convolutional networks.
\newblock \emph{TPAMI}, 2015{\natexlab{b}}.

\bibitem[Donti et~al.(2021)Donti, Rolnick, and Kolter]{donti2021dc3}
Priya~L Donti, David Rolnick, and J~Zico Kolter.
\newblock Dc3: A learning method for optimization with hard constraints.
\newblock In \emph{ICLR}, 2021.

\bibitem[Dupont et~al.(2019)Dupont, Doucet, and Teh]{dupont2019augmented}
Emilien Dupont, Arnaud Doucet, and Yee~Whye Teh.
\newblock Augmented neural odes.
\newblock In \emph{NeurIPS}, 2019.

\bibitem[Fu et~al.(2019)Fu, Zha, Wu, Ding, and Paisley]{restor13}
Xueyang Fu, Zheng-Jun Zha, Feng Wu, Xinghao Ding, and John Paisley.
\newblock Jpeg artifacts reduction via deep convolutional sparse coding.
\newblock In \emph{ICCV}, 2019.

\bibitem[Fu et~al.(2021)Fu, Wang, Cao, Ding, and Zha]{restor15}
Xueyang Fu, Menglu Wang, Xiangyong Cao, Xinghao Ding, and Zheng-Jun Zha.
\newblock A model-driven deep unfolding method for jpeg artifacts removal.
\newblock \emph{TNNLS}, 2021.

\bibitem[Fung et~al.(2021)Fung, Heaton, Li, McKenzie, Osher, and
  Yin]{fung2021fixed}
Samy~Wu Fung, Howard Heaton, Qiuwei Li, Daniel McKenzie, Stanley Osher, and
  Wotao Yin.
\newblock Fixed point networks: Implicit depth models with jacobian-free
  backprop.
\newblock \emph{arXiv preprint arXiv:2103.12803}, 2021.

\bibitem[Geng et~al.(2020)Geng, Guo, Chen, Li, Wei, and Lin]{geng2020attention}
Zhengyang Geng, Meng-Hao Guo, Hongxu Chen, Xia Li, Ke Wei, and Zhouchen Lin.
\newblock Is attention better than matrix decomposition?
\newblock In \emph{ICLR}, 2020.

\bibitem[Geng et~al.(2021)Geng, Zhang, Bai, Wang, and Lin]{geng2021training}
Zhengyang Geng, Xin-Yu Zhang, Shaojie Bai, Yisen Wang, and Zhouchen Lin.
\newblock On training implicit models.
\newblock In \emph{NeurIPS}, 2021.

\bibitem[Gu et~al.(2022)Gu, Goel, and R{\'e}]{gu2021efficiently}
Albert Gu, Karan Goel, and Christopher R{\'e}.
\newblock Efficiently modeling long sequences with structured state spaces.
\newblock In \emph{ICLR}, 2022.

\bibitem[Gu et~al.(2020)Gu, Chang, Zhu, Sojoudi, and El~Ghaoui]{gu2020implicit}
Fangda Gu, Heng Chang, Wenwu Zhu, Somayeh Sojoudi, and Laurent El~Ghaoui.
\newblock Implicit graph neural networks.
\newblock In \emph{NeurIPS}, 2020.

\bibitem[Gulrajani et~al.(2017)Gulrajani, Ahmed, Arjovsky, Dumoulin, and
  Courville]{WGAN-GP}
Ishaan Gulrajani, Faruk Ahmed, Martin Arjovsky, Vincent Dumoulin, and Aaron
  Courville.
\newblock Improved training of wasserstein gans.
\newblock In \emph{NeurIPS}, 2017.

\bibitem[Ho et~al.(2020)Ho, Jain, and Abbeel]{ho2020ddpm}
Jonathan Ho, Ajay Jain, and Pieter Abbeel.
\newblock Denoising diffusion probabilistic models.
\newblock In \emph{NeurIPS}, 2020.

\bibitem[Jia et~al.(2019)Jia, Liu, Feng, and Zhang]{restor12}
Xixi Jia, Sanyang Liu, Xiangchu Feng, and Lei Zhang.
\newblock Focnet: A fractional optimal control network for image denoising.
\newblock In \emph{CVPR}, 2019.

\bibitem[Karras et~al.(2018)Karras, Aila, Laine, and
  Lehtinen]{karras2018progressive}
Tero Karras, Timo Aila, Samuli Laine, and Jaakko Lehtinen.
\newblock Progressive growing of {GAN}s for improved quality, stability, and
  variation.
\newblock In \emph{ICLR}, 2018.

\bibitem[Kawar et~al.(2021)Kawar, Vaksman, and Elad]{kawar2021snips}
Bahjat Kawar, Gregory Vaksman, and Michael Elad.
\newblock Snips: Solving noisy inverse problems stochastically.
\newblock \emph{NeurIPS}, 2021.

\bibitem[Kawar et~al.(2022)Kawar, Elad, Ermon, and Song]{kawar2022ddrm}
Bahjat Kawar, Michael Elad, Stefano Ermon, and Jiaming Song.
\newblock Denoising diffusion restoration models.
\newblock In \emph{NeurIPS}, 2022.

\bibitem[Kim et~al.(2016)Kim, Lee, and Lee]{VDSR}
Jiwon Kim, Jung~Kwon Lee, and Kyoung~Mu Lee.
\newblock Accurate image super-resolution using very deep convolutional
  networks.
\newblock In \emph{CVPR}, 2016.

\bibitem[Kim et~al.(2019)Kim, Soh, Park, Ahn, Lee, Moon, and Cho]{restor14}
Yoonsik Kim, Jae~Woong Soh, Jaewoo Park, Byeongyong Ahn, Hyun-Seung Lee,
  Young-Su Moon, and Nam~Ik Cho.
\newblock A pseudo-blind convolutional neural network for the reduction of
  compression artifacts.
\newblock \emph{TCSVT}, 2019.

\bibitem[Kingma et~al.(2021)Kingma, Salimans, Poole, and
  Ho]{kingma2021variational}
Diederik Kingma, Tim Salimans, Ben Poole, and Jonathan Ho.
\newblock Variational diffusion models.
\newblock In \emph{NeurIPS}, 2021.

\bibitem[Ledig et~al.(2017)Ledig, Theis, Husz{\'a}r, Caballero, Cunningham,
  Acosta, Aitken, Tejani, Totz, Wang, et~al.]{ledig2017srgan}
Christian Ledig, Lucas Theis, Ferenc Husz{\'a}r, Jose Caballero, Andrew
  Cunningham, Alejandro Acosta, Andrew Aitken, Alykhan Tejani, Johannes Totz,
  Zehan Wang, et~al.
\newblock Photo-realistic single image super-resolution using a generative
  adversarial network.
\newblock In \emph{CVPR}, 2017.

\bibitem[Li et~al.(2022{\natexlab{a}})Li, Wang, and Lin]{li2022cerdeq}
Mingjie Li, Yisen Wang, and Zhouchen Lin.
\newblock Cerdeq: Certifiable deep equilibrium model.
\newblock In \emph{ICML}, 2022{\natexlab{a}}.

\bibitem[Li et~al.(2022{\natexlab{b}})Li, Lin, Zhou, Qi, Wang, and Jia]{MAT}
Wenbo Li, Zhe Lin, Kun Zhou, Lu Qi, Yi Wang, and Jiaya Jia.
\newblock Mat: Mask-aware transformer for large hole image inpainting.
\newblock In \emph{CVPR}, 2022{\natexlab{b}}.

\bibitem[Liang et~al.(2021)Liang, Cao, Sun, Zhang, Van~Gool, and
  Timofte]{SWINIR}
Jingyun Liang, Jiezhang Cao, Guolei Sun, Kai Zhang, Luc Van~Gool, and Radu
  Timofte.
\newblock Swinir: Image restoration using swin transformer.
\newblock In \emph{ICCVW}, 2021.

\bibitem[Liang et~al.(2024)Liang, Cao, Fan, Zhang, Ranjan, Li, Timofte, and
  Van~Gool]{liang2022vrt}
Jingyun Liang, Jiezhang Cao, Yuchen Fan, Kai Zhang, Rakesh Ranjan, Yawei Li,
  Radu Timofte, and Luc Van~Gool.
\newblock Vrt: A video restoration transformer.
\newblock \emph{TIP}, 2024.

\bibitem[Lin et~al.(2023)Lin, He, Chen, Lyu, Fei, Dai, Ouyang, Qiao, and
  Dong]{lin2023diffbir}
Xinqi Lin, Jingwen He, Ziyan Chen, Zhaoyang Lyu, Ben Fei, Bo Dai, Wanli Ouyang,
  Yu Qiao, and Chao Dong.
\newblock Diffbir: Towards blind image restoration with generative diffusion
  prior.
\newblock \emph{arXiv preprint arXiv:2308.15070}, 2023.

\bibitem[Lu et~al.(2021)Lu, Chen, Li, Wang, and Zhu]{lu2021implicit}
Cheng Lu, Jianfei Chen, Chongxuan Li, Qiuhao Wang, and Jun Zhu.
\newblock Implicit normalizing flows.
\newblock In \emph{ICLR}, 2021.

\bibitem[Lu et~al.(2022)Lu, Zhou, Bao, Chen, Li, and Zhu]{lu2022dpm}
Cheng Lu, Yuhao Zhou, Fan Bao, Jianfei Chen, Chongxuan Li, and Jun Zhu.
\newblock Dpm-solver: A fast ode solver for diffusion probabilistic model
  sampling in around 10 steps.
\newblock In \emph{NeurIPS}, 2022.

\bibitem[Lugmayr et~al.(2022)Lugmayr, Danelljan, Romero, Yu, Timofte, and
  Van~Gool]{lugmayr2022repaint}
Andreas Lugmayr, Martin Danelljan, Andres Romero, Fisher Yu, Radu Timofte, and
  Luc Van~Gool.
\newblock Repaint: Inpainting using denoising diffusion probabilistic models.
\newblock In \emph{CVPR}, 2022.

\bibitem[Menon et~al.(2020)Menon, Damian, Hu, Ravi, and Rudin]{menon2020pulse}
Sachit Menon, Alexandru Damian, Shijia Hu, Nikhil Ravi, and Cynthia Rudin.
\newblock Pulse: Self-supervised photo upsampling via latent space exploration
  of generative models.
\newblock In \emph{CVPR}, 2020.

\bibitem[Pan et~al.(2021)Pan, Zhan, Dai, Lin, Loy, and Luo]{pan2021dgp}
Xingang Pan, Xiaohang Zhan, Bo Dai, Dahua Lin, Chen~Change Loy, and Ping Luo.
\newblock Exploiting deep generative prior for versatile image restoration and
  manipulation.
\newblock \emph{TPAMI}, 2021.

\bibitem[Pathak et~al.(2016)Pathak, Krahenbuhl, Donahue, Darrell, and
  Efros]{inpainting-GAN}
Deepak Pathak, Philipp Krahenbuhl, Jeff Donahue, Trevor Darrell, and Alexei~A
  Efros.
\newblock Context encoders: Feature learning by inpainting.
\newblock In \emph{CVPR}, 2016.

\bibitem[Pokle et~al.(2022)Pokle, Geng, and Kolter]{pokle2022deqddim}
Ashwini Pokle, Zhengyang Geng, and J~Zico Kolter.
\newblock Deep equilibrium approaches to diffusion models.
\newblock In \emph{NeurIPS}, 2022.

\bibitem[Rombach et~al.(2022)Rombach, Blattmann, Lorenz, Esser, and
  Ommer]{rombach2022ldm}
Robin Rombach, Andreas Blattmann, Dominik Lorenz, Patrick Esser, and Bj{\"o}rn
  Ommer.
\newblock High-resolution image synthesis with latent diffusion models.
\newblock In \emph{CVPR}, 2022.

\bibitem[Saharia et~al.(2022{\natexlab{a}})Saharia, Chan, Chang, Lee, Ho,
  Salimans, Fleet, and Norouzi]{saharia2022palette}
Chitwan Saharia, William Chan, Huiwen Chang, Chris Lee, Jonathan Ho, Tim
  Salimans, David Fleet, and Mohammad Norouzi.
\newblock Palette: Image-to-image diffusion models.
\newblock In \emph{ACM SIGGRAPH}, 2022{\natexlab{a}}.

\bibitem[Saharia et~al.(2022{\natexlab{b}})Saharia, Ho, Chan, Salimans, Fleet,
  and Norouzi]{saharia2022sr3}
Chitwan Saharia, Jonathan Ho, William Chan, Tim Salimans, David~J Fleet, and
  Mohammad Norouzi.
\newblock Image super-resolution via iterative refinement.
\newblock \emph{TPAMI}, 2022{\natexlab{b}}.

\bibitem[Song et~al.(2021{\natexlab{a}})Song, Meng, and Ermon]{song2020ddim}
Jiaming Song, Chenlin Meng, and Stefano Ermon.
\newblock Denoising diffusion implicit models.
\newblock In \emph{ICLR}, 2021{\natexlab{a}}.

\bibitem[Song et~al.(2021{\natexlab{b}})Song, Sohl-Dickstein, Kingma, Kumar,
  Ermon, and Poole]{song2020score}
Yang Song, Jascha Sohl-Dickstein, Diederik~P Kingma, Abhishek Kumar, Stefano
  Ermon, and Ben Poole.
\newblock Score-based generative modeling through stochastic differential
  equations.
\newblock In \emph{ICLR}, 2021{\natexlab{b}}.

\bibitem[Wang et~al.(2023{\natexlab{a}})Wang, Yue, Zhou, Chan, and
  Loy]{wang2023stablesr}
Jianyi Wang, Zongsheng Yue, Shangchen Zhou, Kelvin~CK Chan, and Chen~Change
  Loy.
\newblock Exploiting diffusion prior for real-world image super-resolution.
\newblock In \emph{arXiv preprint arXiv:2305.07015}, 2023{\natexlab{a}}.

\bibitem[Wang et~al.(2023{\natexlab{b}})Wang, Teng, and Wang]{wang2023deep}
Shuai Wang, Yao Teng, and Limin Wang.
\newblock Deep equilibrium object detection.
\newblock In \emph{ICCV}, 2023{\natexlab{b}}.

\bibitem[Wang et~al.(2020)Wang, Zhang, and Sun]{wang2020implicit}
Tiancai Wang, Xiangyu Zhang, and Jian Sun.
\newblock Implicit feature pyramid network for object detection.
\newblock \emph{arXiv preprint arXiv:2012.13563}, 2020.

\bibitem[Wang et~al.(2018)Wang, Yu, Wu, Gu, Liu, Dong, Qiao, and
  Change~Loy]{ESRGAN}
Xintao Wang, Ke Yu, Shixiang Wu, Jinjin Gu, Yihao Liu, Chao Dong, Yu Qiao, and
  Chen Change~Loy.
\newblock Esrgan: Enhanced super-resolution generative adversarial networks.
\newblock In \emph{ECCVW}, 2018.

\bibitem[Wang et~al.(2021)Wang, Xie, Dong, and Shan]{Real-ESRGAN}
Xintao Wang, Liangbin Xie, Chao Dong, and Ying Shan.
\newblock Real-esrgan: Training real-world blind super-resolution with pure
  synthetic data.
\newblock In \emph{ICCVW}, 2021.

\bibitem[Wang et~al.(2023{\natexlab{c}})Wang, Yu, and Zhang]{wang2022ddnm}
Yinhuai Wang, Jiwen Yu, and Jian Zhang.
\newblock Zero-shot image restoration using denoising diffusion null-space
  model.
\newblock In \emph{ICLR}, 2023{\natexlab{c}}.

\bibitem[Wei and Kolter(2021)]{wei2021certified}
Colin Wei and J~Zico Kolter.
\newblock Certified robustness for deep equilibrium models via interval bound
  propagation.
\newblock In \emph{ICLR}, 2021.

\bibitem[Whang et~al.(2022)Whang, Delbracio, Talebi, Saharia, Dimakis, and
  Milanfar]{whang2022deblurring}
Jay Whang, Mauricio Delbracio, Hossein Talebi, Chitwan Saharia, Alexandros~G
  Dimakis, and Peyman Milanfar.
\newblock Deblurring via stochastic refinement.
\newblock In \emph{CVPR}, 2022.

\bibitem[Xia et~al.(2022)Xia, Hang, Tian, Yang, Liao, and Zhou]{ENLCA}
Bin Xia, Yucheng Hang, Yapeng Tian, Wenming Yang, Qingmin Liao, and Jie Zhou.
\newblock Efficient non-local contrastive attention for image super-resolution.
\newblock In \emph{AAAI}, 2022.

\bibitem[Xia et~al.(2023{\natexlab{a}})Xia, Zhang, Wang, Wang, Wu, Tian, Yang,
  and Van~Gool]{xia2023diffir}
Bin Xia, Yulun Zhang, Shiyin Wang, Yitong Wang, Xinglong Wu, Yapeng Tian,
  Wenming Yang, and Luc Van~Gool.
\newblock Diffir: Efficient diffusion model for image restoration.
\newblock In \emph{ICCV}, 2023{\natexlab{a}}.

\bibitem[Xia et~al.(2023{\natexlab{b}})Xia, Zhang, Wang, Tian, Yang, Timofte,
  and Van~Gool]{KDSR}
Bin Xia, Yulun Zhang, Yitong Wang, Yapeng Tian, Wenming Yang, Radu Timofte, and
  Luc Van~Gool.
\newblock Knowledge distillation based degradation estimation for blind
  super-resolution.
\newblock In \emph{ICLR}, 2023{\natexlab{b}}.

\bibitem[Xie et~al.(2019)Xie, Liu, Li, Cheng, Zuo, Liu, Wen, and
  Ding]{inpainting3}
Chaohao Xie, Shaohui Liu, Chao Li, Ming-Ming Cheng, Wangmeng Zuo, Xiao Liu,
  Shilei Wen, and Errui Ding.
\newblock Image inpainting with learnable bidirectional attention maps.
\newblock In \emph{ICCV}, 2019.

\bibitem[Yang et~al.(2022)Yang, Pang, and Liu]{yang2022closer}
Zonghan Yang, Tianyu Pang, and Yang Liu.
\newblock A closer look at the adversarial robustness of deep equilibrium
  models.
\newblock In \emph{NeurIPS}, 2022.

\bibitem[Yi et~al.(2020)Yi, Tang, Azizi, Jang, and Xu]{inpainting2}
Zili Yi, Qiang Tang, Shekoofeh Azizi, Daesik Jang, and Zhan Xu.
\newblock Contextual residual aggregation for ultra high-resolution image
  inpainting.
\newblock In \emph{CVPR}, 2020.

\bibitem[Yu et~al.(2018)Yu, Lin, Yang, Shen, Lu, and Huang]{inpainting1}
Jiahui Yu, Zhe Lin, Jimei Yang, Xiaohui Shen, Xin Lu, and Thomas~S Huang.
\newblock Generative image inpainting with contextual attention.
\newblock In \emph{CVPR}, 2018.

\bibitem[Yu et~al.(2019)Yu, Lin, Yang, Shen, Lu, and Huang]{deepfillv2}
Jiahui Yu, Zhe Lin, Jimei Yang, Xiaohui Shen, Xin Lu, and Thomas~S Huang.
\newblock Free-form image inpainting with gated convolution.
\newblock In \emph{ICCV}, 2019.

\bibitem[Yue and Loy(2022)]{yue2022difface}
Zongsheng Yue and Chen~Change Loy.
\newblock Difface: Blind face restoration with diffused error contraction.
\newblock \emph{arXiv preprint arXiv:2212.06512}, 2022.

\bibitem[Zamir et~al.(2022)Zamir, Arora, Khan, Hayat, Khan, and
  Yang]{restormer}
Syed~Waqas Zamir, Aditya Arora, Salman Khan, Munawar Hayat, Fahad~Shahbaz Khan,
  and Ming-Hsuan Yang.
\newblock Restormer: Efficient transformer for high-resolution image
  restoration.
\newblock In \emph{CVPR}, 2022.

\bibitem[Zeng et~al.(2022)Zeng, Fu, Chao, and Guo]{AOTGAN}
Yanhong Zeng, Jianlong Fu, Hongyang Chao, and Baining Guo.
\newblock Aggregated contextual transformations for high-resolution image
  inpainting.
\newblock \emph{TVCG}, 2022.

\bibitem[Zhang et~al.(2021{\natexlab{a}})Zhang, Li, Zuo, Zhang, Van~Gool, and
  Timofte]{plug-denoiser}
Kai Zhang, Yawei Li, Wangmeng Zuo, Lei Zhang, Luc Van~Gool, and Radu Timofte.
\newblock Plug-and-play image restoration with deep denoiser prior.
\newblock \emph{TPAMI}, 2021{\natexlab{a}}.

\bibitem[Zhang et~al.(2021{\natexlab{b}})Zhang, Liang, Van~Gool, and
  Timofte]{zhang2021bsrgan}
Kai Zhang, Jingyun Liang, Luc Van~Gool, and Radu Timofte.
\newblock Designing a practical degradation model for deep blind image
  super-resolution.
\newblock In \emph{ICCV}, 2021{\natexlab{b}}.

\bibitem[Zhang et~al.(2016)Zhang, Isola, and Efros]{zhang2016colorful}
Richard Zhang, Phillip Isola, and Alexei~A Efros.
\newblock Colorful image colorization.
\newblock In \emph{ECCV}, 2016.

\bibitem[Zhang et~al.(2018)Zhang, Li, Li, Wang, Zhong, and Fu]{RCAN}
Yulun Zhang, Kunpeng Li, Kai Li, Lichen Wang, Bineng Zhong, and Yun Fu.
\newblock Image super-resolution using very deep residual channel attention
  networks.
\newblock In \emph{ECCV}, 2018.

\bibitem[Zhou et~al.(2022)Zhou, Chan, Li, and Loy]{zhou2022codeformer}
Shangchen Zhou, Kelvin~C.K. Chan, Chongyi Li, and Chen~Change Loy.
\newblock Towards robust blind face restoration with codebook lookup
  transformer.
\newblock In \emph{NeurIPS}, 2022.

\bibitem[Zhu et~al.(2023)Zhu, Zhang, Liang, Cao, Wen, Timofte, and
  Van~Gool]{zhu2023diffpir}
Yuanzhi Zhu, Kai Zhang, Jingyun Liang, Jiezhang Cao, Bihan Wen, Radu Timofte,
  and Luc Van~Gool.
\newblock Denoising diffusion models for plug-and-play image restoration.
\newblock In \emph{CVPRW}, 2023.

\end{thebibliography}
}

\end{document}